\documentclass{article}

\usepackage{fullpage}
\usepackage{hyperref}       
\usepackage{url}            
\usepackage{booktabs}       
\usepackage{amsfonts}       
\usepackage{nicefrac}       
\usepackage{microtype}      
\usepackage{xcolor}         
\usepackage{paralist,amsmath, amssymb}
\usepackage{algorithm,algorithmic}
 \usepackage{graphicx}
 \usepackage{float}
 \usepackage{pdfpages}
\usepackage{subfigure} 
\usepackage{xcolor}
\usepackage{graphics}
\usepackage{natbib}
\usepackage{units}
 \usepackage{amssymb, multirow, paralist, color}
 \usepackage{amsthm}
 \usepackage{textcase, booktabs,makecell}
 \usepackage{thmtools,thm-restate}
 \usepackage{mathtools}
\usepackage{enumitem}
\usepackage{titlesec}
\usepackage{tikz}
\usetikzlibrary{positioning}
\usepackage{algorithm,algorithmic}
 \usepackage{thmtools,thm-restate}
 \usepackage{mathtools}
\usepackage{enumitem}

\def \E {\mathbb{E}}

\def \R {\mathbb{R}}

\def \P {\mathbb{P}}
\def \II {\mathbb{I}}
\def \FTL {\textsf{FTL}}
\def \FPL {\textsf{GFTPL}}
\def \OFF {\textsf{OFF}}
\def \Alg {\textsf{Alg}}

\def \W {\mathcal{W}}
\def \N {\mathbb{N}}
\def \A {\mathcal{A}}

\def \I {\mathcal{I}}

\def \B {\mathbf{B}}

\def \Y {\mathcal{Y}}

\def \V {\mathcal{V}}

\def \V {\mathcal{V}}

\def \B {\mathcal{B}}
\def \X {\mathcal{X}}
\def \S {\mathcal{S}}

\newtheorem{ass}{Assumption}

\newtheorem{cor}{Corollary}
\newtheorem{theorem}{Theorem}
\newtheorem{lemma}{Lemma}
\newtheorem{defn}{Definition}

\DeclareMathOperator*{\argmin}{argmin}

\usepackage{microtype}
\usepackage{graphicx}
\usepackage{subfigure}
\usepackage{booktabs}

\begin{document}

\begin{center}

{\bf{\Large{Adaptive Oracle-Efficient Online Learning}}}

\vspace*{.2in}

{\large{
\begin{tabular}{cccc}
  Guanghui Wang$^{\dagger}$ & Zihao Hu$^{\dagger}$ & Vidya Muthukumar$^{\ddagger}$ & Jacob Abernethy$^{\dagger}$  \\
\end{tabular}}}

\vspace*{.05in}

\begin{tabular}{c}
College of Computing$^{\dagger}$\\
School of Electrical and Computer Engineering$^{\ddagger}$\\
School of Industrial and Systems Engineering$^{\ddagger}$\\
 Georgia Institute of Technology\\
  Atlanta, GA 30339 \\
  \texttt{\{gwang369,zihaohu,vmuthukumar8,prof\}@gatech.edu} \\
\end{tabular}

\vspace*{.2in}
\date{}

\end{center}




\begin{abstract}
\noindent 
The classical algorithms for online learning and decision-making have the benefit of achieving the optimal performance guarantees, but suffer from computational complexity limitations when implemented at scale. More recent sophisticated techniques, which we refer to as \textit{oracle-efficient} methods, address this problem by dispatching to an \textit{offline optimization oracle} that can search through an exponentially-large (or even infinite) space of decisions and select that which performed the best on any dataset. But despite the benefits of computational feasibility, oracle-efficient algorithms exhibit one major limitation: while performing well in worst-case settings, they do not adapt well to friendly environments. In this paper we consider two such friendly scenarios, (a) ``small-loss'' problems and (b) IID data. We provide a new framework for designing follow-the-perturbed-leader algorithms that are oracle-efficient and adapt well to the small-loss environment, under a particular condition which we call \textit{approximability} (which is spiritually related to sufficient conditions provided in~\citep{dudik2020oracle}). We identify a series of real-world settings, including online auctions and transductive online classification, for which approximability holds. We also extend the algorithm to an IID data setting and establish a ``best-of-both-worlds'' bound in the oracle-efficient setting. 
\end{abstract}

\section{Introduction}

Online learning is a fundamental paradigm for modeling sequential decision making problems \citep{bianchi-2006-prediction,Online:suvery,Intro:Online:Convex}. 
Online learning is usually formulated as a zero-sum game between a learner and an adversary.
In each round $t=1,\dots,T$, the learner first picks an action $x_t$ from a (finite) set $\X=\{x^{(1)},\dots,x^{(K)}\}$ with cardinality equal to $K$. In the meantime, an adversary reveals its action $y_t\in\Y$. As a consequence, the learner observes $y_t$, and suffers a loss $f(x_t,y_t)$, where $f:\X\times\Y\mapsto[0,1]$. The goal is to minimize the \emph{regret}, which is defined as the difference between the cumulative loss of the learner $\sum_{t=1}^T f(x_t,y_t)$, and the cumulative loss  of the best action in hindsight $L_T^*=\min_{x\in\X}\sum_{t=1}^T f(x,y_t)$.

A wide variety of algorithms have been proposed for the goal of minimizing worst-case regret (without any consideration of computational complexity per iteration); see~\citep{bianchi-2006-prediction,Online:suvery,Intro:Online:Convex} for representative surveys of this literature.
These algorithms all obtain a worst-case regret bound of the order $O(\sqrt{T \log K})$, which is known to be minimax-optimal~\citep{bianchi-2006-prediction}.
Over the last two decades, sophisticated adaptive algorithms have been designed that additionally enjoy \emph{problem-dependent} performance guarantees,
which can automatically lead to better results in friendly environments. One of the most important example for this kind of guarantees is the so-called ``small-loss'' bound \citep{hutter2005adaptive,bianchi-2006-prediction,van2014follow}. Such a bound  depends on the best cumulative loss in hindsight (i.e. $L_T^*$) instead of the total number of rounds ($T$).
Thus, this bound is much  tighter than the worst-case bound, especially when the best decision performs well in the sense of incurring a very small loss. Another example is the ``best-of-both-worlds'' bound \citep{van2014follow}, which results in an even tighter regret bound for independent and identically distributed (IID) loss functions. 

However, all of these algorithms applied out-of-the-box suffer a \emph{linear} dependence on the number of decisions $K$.  This is prohibitively expensive, especially in problems such as network routing \citep{awerbuch2008online} and combinatorial market design \citep{cesa2014regret}, where the cardinality of the decision set grows exponentially with the natural expression of the problem. Several efficient algorithms do exist, even for uncountably infinite decision sets, when the loss functions have certain special structure (such as linearity~\citep{kalai2005efficient} or convexity~\citep{zinkevich-2003-online}). However, such structure is often absent in the above applications of interests.

Notice that the efficiency of the above specialized methods is usually made possible by assuming that the corresponding \emph{offline} optimization problem (i.e., minimizing the (averaged) loss) can be solved efficiently. This observation motivates the \emph{oracle-efficient online learning} problem \citep{hazan2016computational}.
In this setting, the learner has access to a black-box offline oracle, which, given a real-weighted dataset $\S=\{(w^{(j)},y^{(j)})\}_{j=1}^n$,
can efficiently return the solution to the following problem: 
\begin{equation}
\label{eqn:oracle}
  \argmin\limits_{x\in\X} \sum_{j=1}^n w^{(j)}f(x,y^{(j)}).
\end{equation}
The goal  is to design \emph{oracle-efficient} algorithms which can query the offline-oracle $O(1)$ times each round.  Concrete examples of such an oracle include algorithms for empirical
risk minimization \citep{PRML}, data-driven market design \citep{nisan2007computationally}, and dynamic programming \citep{bertsekas2019reinforcement}.  

As pointed out by \cite{hazan2016computational}, the design of oracle-efficient algorithms is extremely challenging and such an algorithm does not exist in the worst case. Nevertheless, recent work \citep{daskalakis2016learning,syrgkanis2016efficient,dudik2020oracle} has introduced a series of algorithms which are oracle-efficient when certain  sufficient conditions are met. 
Among them, the state-of-the-art method is the \emph{generalized-follow-the-perturbed-leader} algorithm~\citep[GFTPL,][]{dudik2020oracle}, which is a variant of the  classical follow-the-perturbed-leader (FTPL) algorithm \citep{kalai2005efficient}. Similar to FTPL, GFTPL perturbs the cumulative loss of each decision by adding a random variable, and chooses the decision with the smallest perturbed loss as $x_t$. However, the vanilla FTPL perturbs each decision independently, which requires to generate $K$ independent random variables in total. Moreover, the oracle in \eqref{eqn:oracle} can not be applied here since as it cannot handle the perturbation term. To address these limitations, GFTPL only generates a noise vector of low dimension (in particular, much smaller dimension than the size of the decision set) in the beginning, and constructs $K$ \emph{dependent} perturbations based on the multiplication between the noise vector and  a \emph{perturbation translation matrix} (PTM). Therefore, the PTM critically ensures that the computational complexity for the noise generation itself is largely reduced. Furthermore, oracle-efficiency can be achieved by setting the elements in the PTM as carefully designed synthetic losses.
\cite{dudik2020oracle} show that a worst-case optimal regret bound can be obtained when the PTM is \emph{admissible}, i.e., every two rows are substantially distinct. This serves as a \emph{sufficient condition} for achieving oracle-efficiency. 

While these results form a solid foundation for general \emph{worst-case} oracle-efficient online learning, it remains unclear whether \emph{problem-dependent}, or \emph{data-adaptive} bounds are achievable in conjunction with oracle-efficiency.
In other words, the design of a generally applicable oracle-efficient \emph{and} adaptive online learning algorithm has remained open. 
In this paper, we provide an affirmative answer to this problem, and make the following contributions. 
\begin{itemize}
    \item We propose a variant of the GFTPL algorithm~\citep{dudik2020oracle}, and derive a new sufficient condition for 
    ensuring oracle-efficiency while achieving the \emph{small-loss} bound. 
    Our key observation is that while the admissibility condition of the PTM in GFTPL successfully stabilizes the algorithm (by ensuring that $\P[x_t\not=x_{t+1}]$ is small), it does not always enable adaptation.
    We address this challenge via a new condition for PTM, called approximability. This condition ensures 
    a stronger stability measure, i.e., the ratio of $\P[x_t=x^{(i)}]$ and $\P[x_{t+1}=x^{(i)}]$ is upper-bounded by a universal constant for any $i\in[K]$, which is critical for proving the small-loss bound. In summary, we obtain the small-loss bound by equipping GFTPL with an approximable PTM, a data-dependent step-size and Laplace distribution for the perturbation noise. As a result of these changes, our analysis path differs significantly from that of~\cite{dudik2020oracle}. 
    Our new condition of \emph{approximability} is simple and interpretable, and can be easily verified for an arbitrary PTM. 
    It shares both similarities and differences from the \emph{admissibility} condition proposed in~\cite{dudik2020oracle}.
    We demonstrate this through several examples where one of the sufficient conditions holds, but not the other.  
    
    \item We identify a series of real-world applications for which we can construct  approximable PTMs: (a) a series of online auctions problems \citep{dudik2020oracle}; (b) problems with a small adversary action space $|\Y|$ \citep{daskalakis2016learning};  and (c) transductive online classification \citep{syrgkanis2016efficient,dudik2020oracle}. This is the first-time that the small-loss bound is obtained in all of these applications.
    To achieve this, we introduce novel PTMs and analysis for showing the approximability condition on these PTMs.
    \item We achieve the ``best-of-both-worlds'' bound, which enjoys even tighter results when the data is IID or the number of leader changes is small. The main idea is to  combine our proposed algorithm with vanilla FTL leveraging ideas from a meta-algorithm called FlipFlop introduced in~\cite{van2014follow}.   
\end{itemize}
\section{Related Work}
Our work contributes to two bodies of work: oracle-efficient online learning and adaptive online learning. In this section, we briefly review the related work in these areas. 

\subsection{Oracle-efficient online learning}
For oracle-efficient online learning, the pioneering work of \cite{hazan2016computational} points out that oracle-efficient methods do not exist when dealing with general hostile adversaries, which implies that additional assumptions on the problem structure have to be made. \cite{daskalakis2016learning} consider the setting in which the cardinality of the adversary's action set $\Y$ is finite and small, and propose to add a series of ``fake'' losses to the learning history based on random samples from $\Y$. They prove that for this setting an $O(|\Y|\sqrt{T})$ regret bound can be obtained.  \cite{syrgkanis2016efficient} study the contextual combinatorial online learning problem, where each action is associated with a binary vector. They make the assumption that the loss function set contains all linear functions as a sub-class. The approach in  \cite{syrgkanis2016efficient} constructs a set of synthetic losses for perturbation based on randomly-selected contexts, and achieves worst-case optimal bounds when all the contextual information can be obtained beforehand, or when there exists a small set of contexts that can tell each decision apart. \cite{dudik2020oracle} is the first work to focus on the general non-contextual setting, and propose the generalized FTPL algorithm.
This algorithm generates a small number of random variables at the beginning, and then 
perturbs the learning history via the innter product between the PTM matrix and the random variables. The algorithm can be implemented efficiently by setting the entries of the PTM as carefully designed loss values. \cite{niazadeh2021online} consider a more complicated combinatorial setting where the offline problem is NP-hard, but a robust \emph{approximation} oracle exists. For this case, they propose an online algorithm based on a multiplicative approximation oracle, and prove that it has low approximate regret, which is a measure weaker than regret, since it only compares with a fraction of the cumulative loss of the best decision in hindsight. Note that none of the aforementioned methods can be easily shown to adapt to friendly structure in data.  Recently, several concurrent works~\citep{block2022smoothed,haghtalab2022oracle} investigate how to obtain tighter bounds oracle-efficiently in the \emph{smoothed-analysis} setting where the distribution of data is close to the uniform distribution \citep{rakhlin2011online,haghtalab2022smoothed}. The main focus is to adapt to the VC dimension of the hypothesis class, rather than improve the dependence on the number of rounds $T$. 

In this paper, we mainly focus on the so-called  the learning with expert advice setting \citep{bianchi-2006-prediction}, where the action set is discrete, and the loss can be highly non-convex. On the other hand, efficient algorithms can be obtained even for continuous action sets when the loss functions have certain properties, such as linearity  \citep{kalai2005efficient,hutter2005adaptive,awerbuch2008online}, convexity \citep{zinkevich-2003-online,Hazan:2007:log} or submodularity \citep{Online:Submodular}.  Finally, we note that,  in this paper we mainly focus on the full-information setting, where the learner can observe the whole loss function after the action is submitted. Oracle-efficient online learning has also been  widely studied in the contextual bandit setting \citep{NIPS2007_3178,10.5555/3020548.3020569,agarwal2014taming,foster2018practical,foster2020beyond}. 
The nature of the oracle-efficient guarantees for the contextual bandit problem is much weaker compared to full-information online learning: positive results either assume a stochastic probability model on the responses given covariates (e.g.~\cite{foster2018practical,foster2020beyond}) or significantly stronger oracles than Eq.~\eqref{eqn:oracle} (e.g.~\cite{agarwal2014taming}).

\subsection{Adaptive online learning}
In this paper, we focus on designing oracle-efficient algorithms with \emph{problem-dependent} regret guarantees. Note that this kind of bound can be achieved by many inefficient algorithms in general, such as Hedge and its variants \citep{bianchi-2006-prediction,de2014follow,pmlr-v40-Luo15}, follow-the-perturbed-leader \citep{kalai2005efficient, van2014follow} or follow-the-regularized-leader \citep{orabona2019modern}. Small-loss bounds can also be obtained efficiently when the loss functions are simply linear \citep{hutter2005adaptive,syrgkanis2016efficient}. On the other hand, in online convex optimization, 
small-loss bounds can be obtained when the loss functions are additionally smooth \citep{NIPS2010_Smooth,Beyond:Logarithmic,AAAI:2020:Wang}. 
However, these algorithms heavily rely on the special structure of the loss functions.
In this paper, we take the first step to extend these methods to support the more complicated (generally non-convex) problems which appear in real-world applications.  

Apart from the small-loss, there exist other types of problem-dependent bounds, such as second-order bound \citep{Cesa-Bianchi2005,gaillard2014second}, quantile bound \citep{chaudhuri2009parameter,pmlr-v40-Koolen15a}, or parameter-free bound \citep{pmlr-v40-Luo15,ssaasdsadas}. Moreover, advanced adaptive results can also be obtained by minimizing  more advanced performances measures other than regret, such as adaptive regret \citep{Adaptive:Hazan,Adaptive:Regret:Smooth:ICML}, or dynamic regret \citep{Adaptive:Dynamic:Regret:NIPS,NEURIPS2020_93931410}. How to obtain these more refined theoretical guarantees in the oracle-efficient setting remains an interesting open problem.

\section{GFTPL with Small-Loss Bound}
\label{section:alg}
In this section, we ignore computational complexity for the moment and 
we provide a new FTPL-type algorithm that enjoys the small-loss bound. We then show that the proposed algorithm can be  implemented  efficiently  by the offline oracle in Section \ref{imp-real}. Before diving into the details, we first briefly recall the definition of online learning and regret.

\paragraph{Preliminaries.} The online decision problem we consider can be described as follows. In each round $t$, a learner picks an action $x_t\in\X=[x^{(1)},\dots,x^{(K)}]$. After observing the adversary's decision $y_t\in\Y$, the learner suffers a loss $f(x_t,y_t)$ where the loss function $f:\X\times\Y\mapsto[0,1]$ is known to the learner and adversary. The regret of an online learning algorithm $\mathcal{A}$ is defined as 
$$ R^{\mathcal{A}}_T := \textstyle \E\left[\sum_{t=1}^T f(x_t,y_t) - L_T^* \right],$$
where $L_T^*=\min\limits_{k\in[K]}\sum_{j=1}^{T}f(x^{(k)},y_j)$ is the cumulative loss of the best action in hindsight, and the expectation is taken only with respect to the potentially randomized strategy of the learner.

Our proposed algorithm follows the framework of GFTPL \citep{dudik2020oracle}. 
We first briefly introduce to the intuition behind this method. Specifically, in each round $t$, GFTPL picks $x_t$ by solving the following optimization problem:
$$ \textstyle x_t = \argmin\limits_{k\in [K]}\sum_{j=1}^{t-1} f(x^{(k)},y_j) + \left<\Gamma^{(k)},\alpha\right>,$$
where $\alpha$ is a $N$-dimensional noise vector ($N\ll K$)  generated from a uniform distribution, and $\Gamma^{(k)}$ is the $k$-th row of a matrix $\Gamma\in[0,1]^{K\times N}$, which is referred to as the \emph{perturbation translation matrix} (PTM). Compared to vanilla FTPL, which generates $K$ random variables (one for each expert), GFTPL only generates $N$ random variables, where $N$ is much smaller than $K$.
Each expert is perturbed by a different linear combination of these random variables based on the PTM $\Gamma$. The results of \cite{dudik2020oracle} rely on the following assumption on $\Gamma$.   
\begin{defn}
    \textbf{{\emph{($\delta$-admissibility \citep{dudik2020oracle}}})} Let $\Gamma\in[0,1]^{K\times N}$ be a  matrix, and denote $\Gamma^{(k)}$ as the $k$-th row of $\Gamma$, and $\Gamma^{(k,i)}$ the $i$-th element of $\Gamma^{(k)}$.
    Then, $\Gamma$ is $\delta$-admissible if (a) $\forall k,k'\in[K]$, $\exists i\in[N]$, such that $\Gamma^{(k,i)}\not=\Gamma^{(k',i)}$; and (b) $\forall i\in[N], k,k'\in[K]$, such that $\Gamma^{(k,i)}\not=\Gamma^{(k',i)}$, then $|\Gamma^{(k,i)}-\Gamma^{(k',i)}|\geq \delta$.
\end{defn}
The $\delta$-admissibility guarantees that every two rows in $\Gamma$ are significantly distinct. As pointed out by \cite{dudik2020oracle}, this is the essential property required by GFTPL, and is used to \emph{stabilize} the algorithm in the analysis, i.e., ensuring that $\P[x_t\not=x_{t+1}]$ is small. However, the adaptive analysis of inefficient FTPL \citep{hutter2005adaptive} (i.e.~using a noise vector of dimension equal to the size of the decision set) reveals that this type of stability is insufficient. Instead, one needs to control the following $\forall t$ and $\forall i\in[K]$,
\begin{equation}
\label{eqn:key-ratio}
    \frac{\P[x_t=x^{(i)}]}{\P[x_{t+1}=x^{(i)}]},
\end{equation}
the ratio of the probability of picking the $i$-th decision in two consecutive rounds. We note that $\delta$-admissibility is not sufficient to ensure this quantity is bounded, as we establish in the following counter-example lemma. (See  Appendix \ref{app:proof:lemma:1} for proof). 
\begin{lemma}
\label{lem:admis}
There is an instance of a $\delta$-admissible $\Gamma$, and a sequence $\{y_t: t=1, 2, \ldots \}$, such that if we run GFTPL we can have $\frac{\P[x_t=x^{(i)}]}{\P[x_{t+1}=x^{(i)}]}=\infty$ for some $i \in [K]$ and some $t > 0$.
\end{lemma}
To address this problem, we propose a new property for $\Gamma$. Define $B_\gamma^1 := \{ s \in \R^N : \|s\|_1 \leq \gamma \}$ as the $\ell_1$-ball of size $\gamma$. 
\begin{defn}
\label{ass:single-out}
    \textbf{{\emph{($\gamma$-approximability)}}} Let $\Gamma\in[0,1]^{K\times N}$. We say that $\Gamma$ is \emph{$\gamma$-approximable} if $$\forall k \in [K], y \in \Y  \; \exists s \in B_\gamma^1 \, \forall j\in [K]: \quad \quad  \left<\Gamma^{(k)} - \Gamma^{(j)}, s \right> \geq f(x^{(k)},y)-f(x^{(j)},y).$$ 
\end{defn}
It may not be immediately obvious how we arrived at this condition, so let us provide some intuition. The goal of perturbation methods in sequential decision problems, going back to the early work of \citet{hannan1957approximation}, is to ensure that the algorithm is ``hedging'' across all available alternative decisions. A newly observed data point $y$ may make expert $j$ suddenly look more attractive than expert $k$, as we have now introduced a new gap $f(x^{(k)},y)-f(x^{(j)},y)$ in their measured loss values. With this in mind, we say that $\Gamma$ is a ``good'' (i.e. approximable) choice for the PTM, if this gap can be overcome (hedged) by some small (i.e. likely) perturbation $s$, so that $\left<\Gamma^{(k)} - \Gamma^{(j)}, s \right>$ makes up the difference.
The \emph{inequality} makes this property  flexible and much easier to satisfy in real-world applications: we only need the gap approximation from above. Later, we will show that $\gamma$-approximability guarantees the required stability measure in \eqref{eqn:key-ratio}, and thus is critical for the small-loss bound. 

We want to emphasize two final points. First, the $\gamma$-approximability condition is \emph{purely for analysis purposes} and we don't need compute the quantity $s$ in response to $y$ and $k$. Second, much of the computational and decision-theoretic challenges rest heavily on the careful design of $\Gamma$. The PTM allows the algorithm to perform the appropriate hedging across an exponentially-sized set of $K$ experts with only $N \ll K$ dimensions of perturbation. As we demonstrate in the following example, we can always construct a $\gamma$-approximable $\Gamma$, with $N = O(\log K)$, but at the expense of computational efficiency. The proposed $\Gamma$ will not generally be \emph{compatible} with the given oracle, in the sense that the optimization problem underlying GFTPL cannot be written in the form of Eq.~\eqref{eqn:oracle}. In the next section, we will show how to address this problem via another condition on $\Gamma$ called \emph{implementablity.}

\paragraph{Simple Example} For any online learning problem we may construct $\Gamma$ as follows. Let $N := \lceil \log_2 K \rceil$, and define the $k$th row $\Gamma^{(k)}$ to be the \emph{binary representation} of the index $k$, with $+1/-1$ values instead of $0/1$. We claim that this $\Gamma$ is $\gamma$-approximable, for $\gamma = \lceil \log_2 K \rceil$. We can satisfy the condition of Definition~\ref{ass:single-out}, by setting $s = \Gamma^{(k)}$. It is easy to see that for any $j \ne k$ we have $\left<\Gamma^{(k)} - \Gamma^{(j)}, s \right> = \left<\Gamma^{(k)} - \Gamma^{(j)}, \Gamma^{(k)} \right> \geq 2 \geq f(x^{(k)},y)-f(x^{(j)},y)$, where the last inequality holds because $|f(x^{(i)},y)| \leq 1$ for any $i \in [K]$.

\paragraph{Comparison beteeen $\gamma$-approximabilty (this paper) and $\delta$-admissibility~\citep{dudik2020oracle}} We note that, although $\gamma$-approximability leads to a much tighter bound, it is not stronger than $\delta$-admissibility. Instead, they are incomparable conditions. Specifically:
\begin{itemize}
\item In Section \ref{subsection:auction} we demonstrate that when $\Gamma$ is binary, 
admissibility directly leads to approximability. As shown by \cite{dudik2020oracle}, a binary and admissible $\Gamma$ exists in various online auctions problems, including VCG with bidder-specific reserves \citep{roughgarden2019minimizing}, envy-free item pricing 
\citep{guruswami2005profit}, online welfare maximization in multi-unit auction \citep{dobzinski2010mechanisms}, and simultaneous second-price auctions \citep{daskalakis2016learning}. We can directly obtain an approximable $\Gamma$ in such cases.
\item On the other hand, in problems such as level auction \citep{dudik2020oracle}, one can construct both admissible and approximable $\Gamma$, although in completely different ways; we discuss the construction in depth in Section \ref{subsection:auction}.
\item In section \ref{subsection:other:application}, we show that, when the adversary's action space is small, we can always construct a $\gamma$-approximable $\Gamma$, while a $\delta$-admissible $\Gamma$does not exist  in general. 
  \item In  Appendix \ref{app:counter:app}, we show that in some cases a $\delta$-admissible $\Gamma$ can be obtained while $\gamma$-approximability cannot be achieved. 
\end{itemize}
\begin{algorithm}[t]
\caption{Generalized follow-the-perturbed-leader with small-loss bound}
\begin{algorithmic}[1]
\label{alg:GFTPL:main}
\STATE \textbf{Init:} $\Gamma\in[0,1]^{K\times N}$
\STATE Draw IID vector $\alpha=[\alpha^{(1)},\dots,\alpha^{(N)}] \sim \text{Lap}(1)^N$; that is, $p(\alpha^{(i)})=\frac{1}{2}\exp(-|\alpha^{(i)}|)$
\FOR{$t=1,\dots,T$}
\STATE Set $\alpha_t \gets \frac{\alpha}{\eta_t}$, where $\eta_t > 0$ a parameter computed online
\STATE Choose $ \displaystyle x_t \gets \argmin\limits_{k\in [K]}\sum_{j=1}^{t-1} f(x^{(k)},y_j) + \left<\Gamma^{(k)},\alpha_t\right>$
\STATE Observe $y_t$
\ENDFOR
\end{algorithmic}
\end{algorithm}
Equipped with the $\gamma$-approximable PTM, we develop a generalized  follow-the-perturbed-leader algorithm with the Laplace distribution for the noise $\alpha$\footnote{Note that the Laplace distribution is not the unique choice to get the small-loss bound. In  Appendix \ref{app:proof:ellp}, we prove that the $\ell_p$ perturbation 
$
p(\alpha)\propto\exp\left\{-\left(\sum_i |\alpha^{(i)}|^p\right)^{\frac{1}{p}}\right\}
$
indeed works for any $p\geq 1$.} 
and a time-varying step size, which is summarized in Algorithm \ref{alg:GFTPL:main}. 
This choice of Laplace distribution is significantly different from the choice of uniform distribution originally used by GTFPL: it turns out that a continuous distribution is required to satisfy Eq.~\eqref{eqn:key-ratio} and thereby the small-loss bound.
Note that here we ignored the time complexity and only focus on the regret. We will specify how to construct $\Gamma$ in the next section. For the proposed algorithm, we successfully obtain the following  stronger stability property.  
\begin{restatable}{lemma}{primelemma}
\label{lem:main:1111}
Assume $\Gamma$ is $\gamma$-approximable. Let $x_t'=\argmin_{k\in[K]}\sum_{j=1}^t f(x^{(k)},y_j)+\left\langle \Gamma^{(k)},\alpha_t \right\rangle$.
Then in each round $t$, we have  $\forall i\in[K],$ 
$$\P[x_t=x^{(i)}]\leq \exp\left({\gamma\eta_t}\right) \P[x'_t=x^{(i)}].$$
\end{restatable}

Note that  we replace the term $x_{t+1}$ in \eqref{eqn:key-ratio} with $x_t'$, as a time-varying step-size is used.  Based on Lemma \ref{lem:main:1111}, we obtain the regret bound of Algorithm \ref{alg:GFTPL:main} as follows.
\begin{theorem}
\label{thm:main:11}
Assume $\Gamma$ is $\gamma$-approximable, and let $L_{T}^*=\min_{k\in[K]}\sum_{j=1}^{T}f(x^{(k)},y_j)$. Algorithm \ref{alg:GFTPL:main}, with $\eta_t=\min\left\{\frac{1}{\gamma},\frac{c}{\sqrt{L^*_{t-1}+1}}\right\}$ for any $c>0$, achieves the following regret bound: 
\begin{equation}
    \begin{split}
R_T\leq {} & \left(\frac{4\sqrt{2}\max\{2\ln K,\sqrt{N\ln K}\}}{c}
+{2\gamma}\left(c+\frac{1}{c}\right)\right) \sqrt{L_T^*+1}\\
{} & + {8\gamma}\ln\left(\frac{1}{c}\sqrt{L_T^*+1}+\gamma\right) + 2\gamma^2 + 4\sqrt{2}\max\{2\ln K,\sqrt{N\ln K}\}\gamma.
  \end{split}
\end{equation}

\end{theorem} 
The proof of Lemma \ref{lem:main:1111} and Theorem \ref{thm:main:11} can be found in Appendix \ref{app:proof:thm1}. By setting $c=\Theta(1)$, Theorem \ref{thm:main:11} implies that our proposed algorithm achieves  $O(\max\{{\gamma},\ln K, \sqrt{N\ln K}\}\sqrt{L_T^*})$ regret bound. 

\paragraph{Comparison to GFTPL \citep{dudik2020oracle}} The original GFTPL algorithm has an $O(\frac{N}{\delta}\sqrt{T})$ regret bound. For the dependence on $T$, our $O(\sqrt{L_T^*})$ bound reduces to $O(\sqrt{T})$ in the worst-case, and automatically becomes tighter when $L_T^*$ is small. On the other hand, for the dependence on other terms, we note that
both $\frac{N}{\delta}$ and $\max\{{\gamma},\ln K, \sqrt{N\ln K}\}$ are lower bounded by $\Omega(\ln K)$, and their exact relationship depends on the specific problem. In Section \ref{imp-real}, we show that for many auction applications, the two terms are on the same order. Moreover, in cases such as when $|\Y|$ is small,  Algorithm \ref{alg:GFTPL:main} with an appropriate $c$ leads to  $O\left(\sqrt{L_T^*\max\{\ln K,\sqrt{|\Y|\ln K}\}}\right)$ regret bound, while the regret bound of GFTPL in \cite{dudik2020oracle} can blow up since $\delta$ can be infinitely small.

 


\section{Oracle-efficiency and Applications}
\label{imp-real}
In this section, we discuss how to run Algorithm \ref{alg:GFTPL:main} in an oracle-efficient way. Following \cite{dudik2020oracle}, we introduce the following  definition. 
\begin{defn}[Implementability]
\label{defn:imp}
A matrix $\Gamma$ is implementable with complexity $M$ if for each $j\in[N]$ there exists a dataset $S_j$, with $|S_j|\leq M$, such that $\forall k,k'\in[K]$, 
$$\Gamma^{(k,j)}-\Gamma^{(k',j)}=\sum_{(w,y)\in\S_j}w\left(f(x^{(k)},y)-f(x^{(k')},y)\right).$$
\end{defn}
Based on Definition \ref{defn:imp}, it is easy to get the following theorem, which is similar to Theorem 2.10 of \cite{dudik2020oracle}.
\begin{theorem}
If\ $\Gamma$ is implementable, then Algorithm \ref{alg:GFTPL:main} is oracle-efficient and has a  per-round complexity $O(T+NM)$.
\end{theorem}
In the following sub-sections, we discuss how to construct approximable and implementable $\Gamma$ matrices in different applications. 
\subsection{Applications in online auctions}
\label{subsection:auction}
\begin{algorithm}[t]
\caption{Oracle-based GFTPL for the reward feedback}
\label{alg:GFTPL:33}
\begin{algorithmic}[1]
\STATE \textbf{Input:} Data set $S_j$, $j\in[N]$, that implement a matrix  $\Gamma\in[0,1]^{K\times N}$, $\eta_1=\min\{\frac{1}{\gamma},1\}$.
\STATE Draw IID vector $\alpha=[\alpha^{(1)},\dots,\alpha^{(N)}] \sim \text{Lap}(1)^N$
\FOR{$t=1,\dots,T$}
\STATE Choose 
 $\displaystyle x_t \gets \argmin\limits_{k\in [K]}\sum_{j=1}^{t-1} f(x^{(k)},y_j) + \sum_{i=1}^N\frac{\alpha^{(i)}}{\eta_t}\left[\sum_{(w,y)\in \S_i}w\cdot r(x^{(k)},y)\right]$
\STATE Observe $y_t$
\STATE Compute ${\widehat{L}}_{t}^*=\min\limits_{k\in[K]}\sum_{j=1}^{t}f(x^{(k)},y_j)$ by using the oracle
\STATE Set $\eta_{t+1} \gets \min\left\{\frac{1}{\gamma},\frac{1}{\sqrt{\widehat{L}_t^*+1}}\right\}$ 
\ENDFOR
\end{algorithmic}
\end{algorithm}
In this part, we apply Algorithm   \ref{alg:GFTPL:main} to  online auction problems, which is the main focus of \cite{dudik2020oracle}. To deal with this sort of problems, we first transform Algorithm \ref{alg:GFTPL:main} to  online learning with rewards setting, i.e., in each round $t$, after choosing $x_t$, instead of suffering a loss, the learner obtains a reward $r(x_t,y_t)\in[0,1]$. For this case, it is straightforward to see that running Algorithm \ref{alg:GFTPL:main} on a surrogate loss $f(x,y)=1-r(x,y)$ directly leads to the small-loss bound. To proceed, we slightly change this procedure and obtain Algorithm \ref{alg:GFTPL:33}. The main difference is that, we {{implement $\Gamma$ with the reward function $r(x,y)$, instead of the surrogate loss $f(x,y)$}}. This makes the construction of $\Gamma$ much easier. We have the following regret bound for Algorithm \ref{alg:GFTPL:33}.
\begin{cor}
\label{cor:imp}
Let $f(x,y)=1-r(x,y)$. Assume $\Gamma$ is $\gamma$-approximable w.r.t. $f(x,y)$ and implementable with function $r(x,y)$. Then Algorithm \ref{alg:GFTPL:33}  is oracle-efficient and achieves the following regret bound: 
$$R_T= \E\left[G_T^*-\sum_{t=1}^Tr(x_t,y_t)\right]=O\left(\max\left\{{\gamma},\ln K, \sqrt{N\ln K}\right\}\sqrt{T-G_T^*}\right),$$
where $G_T^*=\max_{i\in[K]}\sum_{t=1}^T r(x^{(i)},y_t)$ is the cumulative reward of the best expert.
\end{cor}
 Next, we discuss how to construct the PTM in several auction problems.


\paragraph{Auctions with binary and admissible $\Gamma$.} As shown by \cite{dudik2020oracle}, in many online auction problems, such as the Vickrey-Clarkes-Groves (VCG) mechanism with bidder-specific reserves \citep{roughgarden2019minimizing}, envy-free item pricing 
\citep{guruswami2005profit}, online welfare maximization in multi-unit auction \citep{dobzinski2010mechanisms} and simultaneous second-price auctions \citep{daskalakis2016learning}, there exists a binary PTM which is  $1$-admissible and implementable with $N$ rows where $N\ll K$. For these cases, we have the following lemma. The proof is deferred to Appendix \ref{app:proof:binaryg}.
\begin{lemma}
\label{lem:binarygamma}
Let $\Gamma\in[0,1]^{K\times N}$ be a binary matrix and $1$-admissible, then $\Gamma$ is $N$-approximable.
\end{lemma}

Note that, $\Gamma$ is binary and 1-admissible, so every two rows of Gamma differ by at least one element. This means that $\Gamma$ must, at the very least, include $\Omega (\ln K)$ columns to encode each row.
Combining this fact with Lemma \ref{lem:binarygamma} and Corollary \ref{cor:imp}, we can  obtain 
an $O(N\sqrt{T-L_T^*})$ bound for all of the above problems. Compared to the original GFTPL algorithm, our condition leads to a similar dependence on $N$ and a tighter dependence on $T$ due to the improved small-loss bound. More details about the aforementioned auction problems and corresponding regret bounds can be found in Appendix \ref{app:proof:auction}.

\paragraph{Level auction} The class of level auctions was first introduced by \cite{morgenstern2015pseudo}, and optimizing over this class enables a $(1-\epsilon)$ multiplicative approximation with respect to Myerson's optimal auction when the distribution of each bidder's valuation is independent from others. For this problem, the PTM in \cite{dudik2020oracle} is not easily to be shown approximable. To address this problem, we propose a novel way of constructing an  approximable and implementable PTM. \emph{The key idea is to utilize a coordinate-wise threhold function to implement} $\Gamma$. Note that this kind of function can not be directly obtained. Instead, we create an augmented problem with a surrogate loss to deal with this issue. For level auction with single-item, $n$-bidders, $s$-level and $m$-discretization level, our method enjoys an $O(nsm\sqrt{T-L_T^{*}})$ regret bound, which is tighter than the $O(nm^2\sqrt{T})$ (\emph{note that} $s\leq m$) of the original GFTPL both on its dependence on the number of rounds $T$ and auction parameters $n,s,m$. Due to page limitations, we postpone the detailed problem description and proof to  Appendix \ref{app:proof:level}. 
\subsection{Other applications}
\label{subsection:other:application}
\paragraph{Oracle learning and finite parameter space} In 
many real-world applications, such as security game \citep{balcan2015commitment} and online bidding with finite threshold vectors \citep{daskalakis2016learning}, the decision set $\X$ is extremely large, while the adversary's action set $|\Y|$ is finite and small. For these problems, we can construct an implementable PTM based the following lemma, whose proof can be found in Appendix \ref{app:proof:lem:smally}. 
\begin{lemma}
\label{lem:small:y}
Consider the setting with $|\Y|=d$ ($d\ll K$), then there exists a $1$-approximable and implementable $\Gamma$ with $d$ columns and complexity 1. 
\end{lemma}
Combining Lemma \ref{lem:small:y} and Theorem \ref{thm:main:11}, and configuring $c=\sqrt{\max\{\ln K,\sqrt{d\ln K}\}}$, 
we observe that our algorithm achieves a small-loss bound on the order of $O(\sqrt{\max\{\ln K,\sqrt{d\ln K}}\}\sqrt{L_T^*})$. On the other hand, because of the continuity of the loss functions in this setting, a $\delta$-admissible PTM in general does not exist (as $\delta$ may approach 0). Therefore, our proposed condition not only leads to a tighter bound, but can also solve problems that the original GFTPL \citep{dudik2020oracle} can not handle.

\paragraph{Transductive online classification} Finally, we consider the transductive online classification problem \citep{syrgkanis2016efficient,dudik2020oracle}. In this setting, the decision set $\X$ consists of $K$ binary classifiers. In each round $t$, firstly the adversary picks a feature vector $w_t\in \W$, where $|\W|=m$. Then, the learner chooses a classifier $x_t(\cdot)$ from $\X$. After that, the adversary reveals the label $y_t\in\{0,1\}$, and the learner suffers a loss $f(x_t,(w_t,y_t))=\mathbb{I}[x_t(w_t)\not=y_t]$. We assume the problem is transductive, i.e., the learner has access to the adversary's set of vectors at the beginning. For this setting, we achieve the following results (the proof is in Appendix \ref{app:lem:5}).
\begin{lemma}
\label{lem:small}
Consider transductive online classification with $|\W|=m$. Then there exists a 1-approximable and implementable PTM with $m$ columns and complexity 1. Moreover, Algorithm \ref{alg:GFTPL:main} with such a PTM and appropriately chosen parameters achieves  $O(\sqrt{\max\{\ln K,\sqrt{m\ln K}}\}\sqrt{L_T^*})$ regret.
\end{lemma}
\paragraph{Negative implementability} In the this paper we assume that the offline oracle can solve the minimization problem in \eqref{eqn:oracle} given any real-weights. In some cases, the oracle can only accept positive weights. This problem can be solved by constructing negative implementable PTM \citep{dudik2020oracle}. In most of the cases discussed above, negative implementable and approximable PTM exist. This  is formally shown in Appendix \ref{app:negei}.

\section{Best-of-Both-Worlds Bound: Adapting to IID data}
In this section, we switch our focus to adapting between adversarial and stochastic data. While the GFTPL algorithm enjoys an $O(\sqrt{L_T^*})$-type regret bound on adversarial data, it is possible to obtain much better rates on stochastic data. For example, by setting all step sizes $\eta_t$ as $\infty$, Algorithm \ref{alg:GFTPL:main} reduces to the classical FTL algorithm, which  suffers  linear regret in the adversarial setting but enjoys much tighter bounds when the data is IID or number of leader changes is small.
To be more specific, we introduce the following regret bound for FTL.
\begin{lemma}[Lemma 9, \cite{de2014follow}]
\label{lem:ftl}
Let $x_t^{\emph{\FTL}}=\argmin_{i\in[K]} \sum_{s=1}^{t-1}f(x^{(i)},y_s)$ be the output of the FTL algorithm at round $t$, $C_T$ the set of rounds where the leader changes, and $\delta_t=f(x^{\emph{\FTL}}_t,y_t)-(L_t^{*}-L_{t-1}^*)$ the ``mixability gap''\footnote{Here, we use the special definition of the mixability gap for the FTL algorithm. The details can be found in the second paragraph, page 1286 of \cite{de2014follow}.} at round $t$. Then for any $T\geq1$, the regret of FTL is bounded by 
$ R_T^{\emph{\FTL}}\leq \sum_{t\in C_T}\delta_t\leq |C_T|.$
\end{lemma}
Note that since $f\in[0,1]$ and $L_t^*-L_{t-1}^*\in[0,f(x^{\emph{\FTL}}_t,y_t)]$, we know $\delta_t\in[0,1]$. For the \textit{i.i.d} case, if the mean loss of the best expert is smaller than that of other experts by a constant, then due to the law of large numbers, the number of leader changes would be small, which results in a constant regret bound \citep{de2014follow}.

Our goal is to obtain a "best-of-both-worlds" bound, which can ensure the small-loss bound in general, while automatically leading to tighter bounds for IID data like FTL. We will now design an algorithm that achieves such a bound by adaptively choosing between GFTPL and FTL depending on which algorithm appears to be achieving a lower regret. The essence of this idea was first introduced in the FlipFlop algorithm \citep{de2014follow}, who showed best-of-both-worlds bounds in the inefficient case. Our contribution in this section is to adapt this idea to the oracle-efficient setting.
Denote $U_T^{\FPL}$ as the attainable regret bound (as in Theorem \ref{thm:main:11}) for running  Algorithm \ref{alg:GFTPL:main} alone and $U_T^{\emph{\FTL}}=\sum_{t\in C_T}\delta_t$ to be that of FTL. In the following, we develop a new algorithm and prove that it is optimal in both worlds, that is, its regret is on the order of  $O(\min\{U_T^{\FTL},U_T^{\FPL}\})$.

\begin{algorithm}[t]
\caption{Oracle-efficient Flipflop (OFF)}
\textbf{Initialization:} $\textsf{Alg}_1=\textsf{FTL}$
\begin{algorithmic}[1]
\label{alg:GFTPL}
\FOR{$t=1,\dots,T$}
\STATE Get $x_t$ by $\textsf{Alg}_t$, observe $y_t$
\STATE Compute $\widehat{U}_t^{\FTL}$ and $\widehat{U}_t^{\FPL}$
\IF{$\textsf{Alg}_t==$ FTL and $\widehat{U}_t^{\FTL}>\alpha \widehat{U}_t^{\FPL}$}
\STATE $\Alg_{t+1} =$ GFTPL
\ELSIF{$\textsf{Alg}_t==$ GFTPL and $\widehat{U}_t^{\FPL}>\beta \widehat{U}_t^{\FTL}$}
\STATE $\Alg_{t+1} =$FTL
\ENDIF 
\STATE Feed $y_t$ to $\Alg_{t+1}$
\ENDFOR
\end{algorithmic}
\end{algorithm}
The proposed algorithm, named as oracle-efficient flipflop (OFF) algorithm, is summarized in Algorithm \ref{alg:GFTPL}. The core idea is to switch between FTL and GFTPL (Algorthm \ref{alg:GFTPL:main}) based on the comparison of the estimated regret. We optimistically start from FTL. In each round $t$, we firstly pick $x_t$ based on the current algorithm $\Alg_t$, and then obtain the adversary's action $y_t$ (line 2). Next, we compute the estimated bounds of regret of both algorithms until round $t$ (line 3). Specifically, let $\I_t^{\FTL}=\{i|i\in[t], \Alg_i=\text{FTL}\}$ and $\I_t^{\FPL}=\{i|i\in[t], \Alg_i=\text{GFTPL}\}$
be the set of rounds up to $t$ in which we run FTL and GFTPL. Then, the estimated regret of FTL in $\I_t^{\FTL}$ is given by 
$\widehat{U}_t^{\FTL}=\sum_{i\in\I_t^{\FTL}}\delta_i,$
and 
the estimated regret of GFTPL in $\I_t^{\FPL}$ can be bounded via Theorem \ref{thm:main:11}: 
\begin{equation}
    \begin{split}
\widehat{U}^{\FPL}_t= {} & \left({4\sqrt{2}\max\{2\ln K,\sqrt{N\ln K}\}}
+{4\gamma}\right) \sqrt{\widehat{L}_t^*+1}\\
{} & + {8\gamma}\ln\left(\sqrt{\widehat{L}_t^*+1}+\gamma\right) + 2\gamma^2 + 4\sqrt{2}\max\{2\ln K,\sqrt{N\ln K}\}\gamma.
  \end{split}
\end{equation}
where $\widehat{L}_t^*=\min_{x\in \X}\sum_{i\in{\I}_t^{\FPL}}f(x,y_i)$ and we set $c=1$. Note that, the two quantities defined above are the exact regret upper bounds of the two algorithms on their sub-time intervals up to round $t$, due to the fact that the regret bounds provided in Lemma \ref{lem:ftl} and Theorem \ref{thm:main:11} are \emph{timeless}. Moreover, note that the two values can be computed by the oracle. We compare the estimated regret of both algorithms, and use the algorithm which performs better for the next round (lines 4-8).  

For the proposed algorithm, we have the following theoretical guarantee (the proof can be found in Appendix \ref{E}). 
\begin{theorem}
\label{thm:off}
Assume we have a $\gamma$-approximable $\Gamma$, then Algorithm \ref{alg:GFTPL} is able to achieve the following bound: 
$$ R_T^{\OFF}\leq \min\left\{3 U_T^{\FPL}+1, 3{U}_T^{\FTL}+\tau\right\},$$
where $\tau={4\sqrt{2}\max\{2\ln K,\sqrt{N\ln K}\}}
+{12\gamma}$ and $\alpha=\beta=1$.
\end{theorem}
The Theorem above shows that the regret of Algorithm \ref{alg:GFTPL} is the minimum of the regret upper bounds of GFTPL and FTL. Thus, it 
ensures the $O(\sqrt{T})$-type bound in the worst case, while 
automatically achieves the much better constant regret bound of FTL under iid data without knowing the presence of stochasticity in data beforehand. 
\section{Conclusion} 
In this paper, we establish a sufficient condition for  the first-order bound in the oracle-efficient setting by investigating a variant of the  generalized  follow-the-perturbed-leader algorithm.
 We also show  the condition is satisfied in various applications. Finally, we extend the algorithm to adapt to IID losses and achieve a ``best-of-both-worlds'' bound. In the future, we would like to investigate
 how to achieve tighter results for oracle-efficient setting, such as the second-order bound \citep{de2014follow} and the quantile bound \citep{pmlr-v40-Koolen15a}.\\

 \noindent \textbf{Acknowledgments.} We gratefully thank the AI4OPT Institute for funding, as part of NSF Award 2112533. We gratefully acknowledge the NSF for their support through Award IIS-2212182 and Adobe Research for their support through a Data Science Research Award. Part of this work was conducted while the authors were visiting the Simons Institute for the Theory of Computing.

\bibliographystyle{icml2021}
\bibliography{ref}

\newpage

\begin{appendix}


\section{Omitted Proofs from Section \ref{section:alg}}
In this section, we provide the omitted proofs from Section \ref{section:alg}. 

\subsection{Proof of Lemma \ref{lem:admis}}
\label{app:proof:lemma:1}
Recall that in GFTPL \citep{dudik2020oracle}, $x_t$ is picked by solving the following optimization problem:
$$\textstyle x_t=\argmin_{k\in[K]}\sum_{s=1}^{t-1}f(x^{(k)},y_s)+\left\langle \Gamma^{(k)},\alpha\right\rangle,$$
where the entries of the random vector $\alpha$ are sampled from a uniform distribution $\mathcal{U}[0,\beta]^{N}$ for some hyperparameter $\beta>0$. 
Recall that to prove Lemma~\ref{lem:admis}, we wish to find a counterexample of a $\delta$-admissible PTM $\Gamma$ and a sequence $\{y_t: t = 1,2,\ldots,\}$ such that the probability distribution induced by GFTPL yields $\frac{\P[x_t = x^{(i)}]}{\P[x_{t+1} = x^{(i)}]} = \infty$ for some $i \in [K]$ and some $t > 0$.
This precludes the possibility of obtaining a stronger small-loss bound on instances that only satisfy $\delta$-admissibility and no other special properties.

In the following, we show that such a counterexample can be found in this case by exploiting the property of the bounded support of the distribution. We then demonstrate that a similar counterexample exists even if the uniform distribution is replaced by  distributions with  unbounded support. 

\subsubsection{Case 1: Uniform noise distribution (or, more generally, distributions with bounded support)}
Fix $K = 2$ experts and round $t > 0$.
Suppose that we can obtain an implementable $\Gamma=[0,1]^{\top}$ which is $1$-admissible with $N=1$ column. 
In this case, $\alpha$ is a scalar random variable, and we have
\begin{equation*}
\begin{split}
      \P[x_t=x^{(1)}] = {} & \P\left[\sum_{s=1}^{t-1}f(x^{(1)},y_s)\leq \sum_{s=1}^{t-1}f(x^{(2)},y_s) +\alpha \right]\\
      = {} & \P\left[\alpha\geq\sum_{s=1}^{t-1}f(x^{(1)},y_s)- \sum_{s=1}^{t-1}f(x^{(2)},y_s)\right].
\end{split}
\end{equation*}
Similarly, at round $t + 1$ we have
\begin{equation*}
    \P[x_{t+1}=x^{(1)}] =  \P\left[\alpha\geq\sum_{s=1}^{t}f(x^{(1)},y_s)- \sum_{s=1}^{t}f(x^{(2)},y_s)\right].
\end{equation*}
Since the probability density function of $\alpha$ has bounded support, it is straightforward to pick appropriate loss functions such that $\sum_{s=1}^t f(x^{(1)},y_s) - \sum_{s=1}^t f(x^{(2)},y_s)$ lies outside the support of the density function while $\sum_{s=1}^{t-1} f(x^{(1)},y_s) - \sum_{s=1}^{t-1} f(x^{(2)},y_s)$ lies inside the support of the density function.
As a consequence, we get $\P[x_{t+1}=x^{(1)}]=0$ while $\P[x_{t}=x^{(1)}]>0$.

\subsubsection{Case 2: Noise distributions with  unbounded support}\label{sssec:counterexampleunbounded}
One may argue that the above bad case happens mainly because the noise density has  bounded support. 
We now show that such counterexamples can also be constructed when the noise $\alpha$ is generated from distributions with  unbounded support, such as the Laplace distribution --- with a slightly larger number of experts.
Specifically, we consider $K = 3$ experts and the PTM $\Gamma=[0,0.5,1]^{\top}$, which is 0.5-admissible with $N = 1$ column.
Then, we have
\begin{equation*}
    \begin{split}
        {} & \P\left[x_t=x^{(2)}\right]\\
        = {} & \P\left[\sum_{s=1}^{t-1}f(x^{(2)},y_s)+0.5\alpha \leq \sum_{s=1}^{t-1}f(x^{(1)},y_s)\ \text{and}\ \sum_{s=1}^{t-1}f(x^{(2)},y_s)+0.5\alpha \leq \sum_{s=1}^{t-1} f(x^{(3)},y_s)+\alpha\right] \\
        = {} & \P \left[ 2\left(\sum_{s=1}^{t-1}f(x^{(2)},y_s) - \sum_{s=1}^{t-1}f(x^{(3)},y_s)\right)\leq \alpha\leq 2\left(\sum_{s=1}^{t-1}f(x^{(1)},y_s) - \sum_{s=1}^{t-1}f(x^{(2)},y_s)\right)\right]\\
        = {} & \P\left[2\Delta_{23}\leq \alpha \leq 2\Delta_{12}\right],
    \end{split}
\end{equation*}
where we have defined $\Delta_{23}=\sum_{s=1}^{t-1}f(x^{(2)},y_s) - \sum_{s=1}^{t-1}f(x^{(3)},y_s)$, and $\Delta_{12}= \sum_{s=1}^{t-1}f(x^{(1)},y_s) - \sum_{s=1}^{t-1}f(x^{(2)},y_s)$ as shorthand.
Now, we pick loss functions such that $\Delta_{12}>\Delta_{23}$ and $\Delta_{12}-\Delta_{23}=\epsilon<1$. 
Because $\Delta_{12} > \Delta_{23}$ and the distribution of $\alpha$ has infinite support, we have $\P[x_t=x^{(2)}]>0$. 
On the other hand, for round $t+1$ a similar argument yields
\begin{equation*}
    \begin{split}
        {} &\P[x_{t+1}=x^{(2)}] \\
        = {} &\P\left[2\Delta_{23}+2\left(f(x^{(2)},y_{t})-f(x^{(3)},y_{t})\right)\leq\alpha\leq 2\Delta_{12}+2\left(f(x^{(1)},y_{t})-f(x^{(2)},y_{t})\right)\right].
    \end{split}
\end{equation*}
Now, we pick $f(x^{(2)},y_{t})=0.5$, and  $f(x^{(1)},y_{t})=f(x^{(3)},y_{t})=0$. For this choice, we get $$2 \Delta_{23} + 2\left(f(x^{(2)},y_{t})-f(x^{(3)},y_{t})\right) \geq 2\Delta_{12}+2\left(f(x^{(1)},y_{t})-f(x^{(2)},y_{t})\right),$$ which implies that $\P[x_{t+1}=x^{(2)}]=0$.
This in turn implies that $\frac{\P[x_t=x^{(2)}]}{\P[x_{t+1}=x^{(2)}]}=\infty$, completing the proof of the counterexample.
\qed

These counterexamples imply that the condition of $\delta$-admissibility alone on the PTM $\Gamma$ is not sufficient to control the stronger stability measure required for a small-loss bound.
Consequently, new assumptions on $\Gamma$ need to be introduced.  

\subsection{Counterexamples showing that $\delta$-admissiblity does not necessarily lead to \texorpdfstring{$\gamma$}{gamma}-approximability}
\label{app:counter:app}

In this paper, we introduced a new sufficient condition of \emph{$\gamma$-approximability} that implies not only worst-case regret bounds but also regret bounds that adapt to the size of the best loss in hindsight. 
It is natural to ask about the relationship of this sufficient condition with $\delta$-admissibility.
In this section, we show that exist $\delta$-admissible PTMs that do not satisfy $\gamma$-approximability.
(Note that the reverse statement is also true: Lemma~\ref{lem:small:y} constructs $\gamma$-approximable PTMs that are not in general $\delta$-admissible.)

The counterexample is precisely the one used in Section~\ref{sssec:counterexampleunbounded}.
That is, there are $K=3$ experts, and the PTM is given by $\Gamma=[0,0.5,1]^{\top}$.
Note that $\Gamma$ is 0.5-admissible with one column. 
Further, we consider an output $y$ such that $f(x^{(1)},y)=f(x^{(3)},y)=0$ and $f(x^{(2)},y)=1$. 
We proceed to show that this PTM $\Gamma$ is not approximable.
To prove this, note that for some scalar $s$ to satisfy the requisite approximability condition, we need $(\Gamma^{(2)}-\Gamma^{(1)})s= 0.5s\geq1$ \emph{and} $(\Gamma^{(2)}-\Gamma^{(3)})s=-0.5s\geq1$.
This is clearly unsatisfiable by any scalar $s$.
\subsection{Proof of Theorem \ref{thm:main:11}}
\label{app:proof:thm1}
We now provide the detailed proof of Theorem \ref{thm:main:11}. We begin by introducing some notation specific to this proof.
We denote by $\Gamma^{(x_t)}$ the row of $\Gamma$ related to expert $x_t$, and by $\Gamma^*$ the row related to the best-expert-in-hindsight $x^*$. 
Further, $\Gamma^{(k)}$ denotes the row of $\Gamma$ related to expert $x^{(k)}$ and $\Gamma^{(k,i)}$ denotes the $i$-th component of the row $\Gamma^{(k)}$.
We also denote the PDF of the noise vector at round $t$, $\alpha_t$, as $p(\alpha_t)$. 
Finally, the learner's action set is denoted by $\mathcal{X}=\{x^{(1)},\dots,x^{(k)},\dots,x^{(K)}\}$. 

Our proof begins with the framework used by typical FTPL analyses~\citep{hutter2005adaptive,syrgkanis2016efficient,dudik2020oracle}. We first  
divide the regret into two terms:
\begin{equation}
    \begin{split}
    \label{eqn:proof:regret:1}
R_T= {} &\E\left[\sum_{t=1}^Tf(x_t,y_t)-f(x^*,y_t)\right]\\     
= {} & \underbrace{\E\left[\sum_{t=1}^Tf(x_t,y_t)-\sum_{t=1}^Tf(x'_t,y_t)\right]}_{\textsc{Term 1}} + \underbrace{\E\left[\sum_{t=1}^Tf(x'_t,y_t)-\sum_{t=1}^Tf(x^*,y_t)\right]}_{\textsc{Term 2}}.
    \end{split}
\end{equation}
Above, the expectation $\E[\cdot]$ is only with respect to the internal randomness of the learner and $x^*=\argmin_{i\in[K]}\sum_{t=1}^{T}f(x^{(i)},y_t)$ is the best decision in hindsight.  
Further, the expert $$\textstyle x'_t=\argmin_{k\in[K]}\sum_{j=1}^{t}f(x^{(k)},y_j)+ \left<\Gamma^{(k)},\alpha_t\right>$$
is usually referred to as the \emph{infeasible} leader~\citep{hutter2005adaptive} at round $t$, since  $y_t$ can only be obtained  after $x_t$ is chosen. 

Next, we bound the two terms of \eqref{eqn:proof:regret:1} respectively. 
\textsc{Term 1} measures the \emph{stability} of GFTPL by tracking how close its performance is to that of the idealized infeasible leader.
We obtain the following upper bound on \textsc{Term 1} which heavily leverages the key technical Lemma~\ref{lem:main:1111}.
\begin{lemma}
\label{lem:main:term1}
Assume that the PTM $\Gamma$ is $\gamma$-approximable, and  Algorithm \ref{alg:GFTPL:main} is applied with  $\textstyle\eta_t=\min\left\{\frac{1}{\gamma},\frac{c}{\sqrt{L^*_{t-1}+1}}\right\}$, where $L_{t-1}^*=\min_{k\in[K]}\sum_{j=1}^{t-1}f(x^{(k)},y_j)$ and $c>0$ is some universal constant. Then for all $T\geq 1$ we have:
\begin{equation*}
    \begin{split}
\textstyle\emph{\textsc{term 1}} \leq {} &\left(\frac{2\sqrt{2}\max\{2\ln K,\sqrt{N\ln K}\}}{c}
+{2\gamma}\left(c+\frac{1}{c}\right)\right) \sqrt{L_T^*+1}\\
{} & + {8\gamma}\ln\left(\frac{1}{c}\sqrt{L_T^*+1}+\gamma\right) + 2\gamma^2 + 2\sqrt{2}\gamma\max\{2\ln K,\sqrt{N\ln K}\}.
    \end{split}
\end{equation*}
\end{lemma}
Next, \textsc{Term 2} measures the approximation error between the infeasible leader and the true best expert in hindsight.
The following lemma, which is a simple extension of the classical be-the-leader lemma \citep{bianchi-2006-prediction}, bounds \textsc{Term 2}.
\begin{lemma}
\label{lem:main:3333}
Assume that the PTM $\Gamma$ is $\gamma$-approximable. Then, for all $T\geq 1$, we have 
$$\textstyle\emph{\textsc{term 2}} \leq 2\sqrt{2}\max\{2\ln K,\sqrt{N\ln K}\}\left({\gamma}+\frac{1}{c}\sqrt{L_T^*}\right)$$
\end{lemma}
We prove Lemmas \ref{lem:main:1111}, \ref{lem:main:term1}, and \ref{lem:main:3333} in Appendix \ref{app:proof:lemmain}, \ref{B.2} and \ref{B.3} respectively.
The proof of Theorem \ref{thm:main:11} follows by directly combining \eqref{eqn:proof:regret:1}, Lemma \ref{lem:main:term1} and Lemma \ref{lem:main:3333}. 

\subsubsection{Proof of Lemma \ref{lem:main:1111}}
\label{app:proof:lemmain}

Note that we can write
$$\textstyle\E[f(x_t,y_t)]=\sum_{i=1}^{K}f(x^{(i)},y_t)\P[x_t=x^{(i)}].$$
Our approach will relate $\P[x_t=x^{(i)}]$ and $\P[x'_t=x^{(i)}]$ for every $i \in [K]$: at a high level, a similar approach is also used in the analysis of contextual online learning for linear functions by~\cite{syrgkanis2016efficient} (although several other aspects of our analysis are different).
Then, for any fixed choice of $s^{(i)} \in \R^N$ we have
\begin{equation}
    \begin{split}
    \label{eqn:main:1:final}
        {} &\P[x_t=x^{(i)}] \\
        = {} & \int_{\alpha_t}\II\left[\left\{\argmin\limits_{k\in[K]}\sum_{j=1}^{t-1}f(x^{(k)},y_j)+\left<\Gamma^{(k)}, \alpha_t\right>\right\}=x^{(i)}\right] p(\alpha_t)d\alpha_t \\
  = {} & \int_{\alpha_t}\II\left[\left\{\argmin\limits_{k\in[K]}\sum_{j=1}^{t-1}f(x^{(k)},y_j)+\left<\Gamma^{(k)},\alpha_t\right>\right\}=x^{(i)}\right] p\left(\alpha_t-s^{(i)}\right)\frac{p(\alpha_t)}{p\left(\alpha_t-s^{(i)}\right)}d\alpha_t\\
  \stackrel{\rm (1)}{\leq}  {} & \sup_{\beta\in\R^N} \frac{p(\beta)}{p\left(\beta-s^{(i)}\right)} \int_{\alpha_t}\II\left[\left\{\argmin\limits_{k\in[K]}\sum_{k=1}^{t-1}f(x^{(k)},y_j)+\left<\Gamma^{(k)},\alpha_t\right>\right\}=x^{(i)}\right] p\left(\alpha_t-s^{(i)}\right)d\alpha_t\\
  \stackrel{\rm (2)}{\leq}  {} & \exp\left({\eta_t}\|s^{(i)}\|_1\right)\int_{\alpha_t}\II\left[\left\{\argmin\limits_{k\in[K]}\sum_{j=1}^{t-1}f(x^{(k)},y_j)+\left<\Gamma^{(k)},\alpha_t\right>\right\}=x^{(i)}\right] p\left(\alpha_t-s^{(i)}\right)d\alpha_t\\
 =  {} & \exp\left({\eta_t}\|s^{(i)}\|_1\right)\int_{\alpha_t}\II\left[\left\{\argmin\limits_{k\in[K]}\sum_{j=1}^{t-1}f(x^{(k)},y_j)+\left<\Gamma^{(k)},\alpha_t+s^{(i)}\right>\right\} =x^{(i)}\right] p\left(\alpha_t\right)d\alpha_t,
    \end{split}
\end{equation}
Above, $\II[\cdot]$ denotes the indicator function and inequality $\rm(2)$ is based on the fact that for any $\beta\in\R^N$, 
\begin{equation}
\label{eqn:thm:1:main:distributionmove}
    \frac{p(\beta)}{p\left(\beta-s^{(i)}\right)}=\exp\left(\eta_t\left(\left\|\beta-s^{(i)}\right\|_1-\|\beta\|_1\right)\right)\leq \exp\left({\eta_t} \|s^{(i)}\|_1\right),
\end{equation}
and the final equality is because  the support of $\alpha_t$ is unbounded. 
To proceed, we introduce and prove the following lemma.
\begin{lemma}
\label{lem:main:222}
Suppose $\Gamma$ is $\gamma$-approximable. Then, $\forall i\in[N]$ there exists a vector $s^{(i)} \in \R^N$ such that
\begin{equation}
\begin{split}
    {} &\II\left[\left\{\argmin\limits_{k\in[K]}\sum_{j=1}^{t-1}f(x^{(k)},y_j)+\left<\Gamma^{(k)},\alpha_t\right>+ \left<\Gamma^{(k)},s^{(i)}\right>\right\} =x^{(i)}\right]\\
    \leq {} &  \II\left[\left\{\argmin\limits_{k\in[K]}\sum_{j=1}^{t-1}f(x^{(k)},y_j)+\left<\Gamma^{(k)},\alpha_t\right>+ f(x^{(k)},y_t)\right\} =x^{(i)}\right]  
\end{split}
\end{equation}
holds  for all $\alpha_t$. 
\end{lemma}
\begin{proof}
For any fixed $\alpha_t$, if $$\textstyle\II\left[\left\{\argmin\limits_{k\in[K]}\sum_{j=1}^{t-1}f(x^{(k)},y_j)+\left<\Gamma^{(k)},\alpha_t\right>+ f(x^{(k)},y_t)\right\} =x^{(i)}\right]=1,$$then the required inequality always holds since the indicator function is  upper bounded by $1$. For the case when  $$\textstyle\II\left[\left\{\argmin\limits_{k\in[K]}\sum_{j=1}^{t-1}f(x^{(k)},y_j)+\left<\Gamma^{(k)},\alpha_t\right>+ f(x^{(k)},y_t)\right\} =x^{(i)}\right]=0,$$ 
assume that $x^{(\ell)}=\textstyle\argmin\limits_{k\in[K]}\sum_{j=1}^{t-1}f(x^{(k)},y_j)+\left<\Gamma^{(k)},\alpha_t\right>+ f(x^{(k)},y_t)$
for some $\ell\not=i$. Then 
\begin{equation}
\label{eqn:proof:of:lemma:3}
    \sum_{j=1}^{t-1}f(x^{(\ell)},y_j)+\left<\Gamma^{(\ell)},\alpha_t\right>+ f(x^{(\ell)},y_t)\leq \sum_{j=1}^{t-1}f(x^{(i)},y_j)+\left<\Gamma^{(i)},\alpha_t\right>+ f(x^{(i)},y_t),
\end{equation}
which implies
\begin{equation}
 \begin{split}
   {} & \sum_{j=1}^{t-1}f(x^{(\ell)},y_j)+\left<\Gamma^{(\ell)},\alpha_t\right>+ \left<\Gamma^{(\ell)},s^{(i)}\right> - 
\left(\sum_{j=1}^{t-1}f(x^{(i)},y_j)+\left<\Gamma^{(i)},\alpha_t\right>+ \left<\Gamma^{(i)},s^{(i)}\right>\right) \\
\stackrel{\rm (1)}{\leq}
 {} &  (f(x^{(i)},y_t) - f(x^{(\ell)},y_t)) + \left(\left<\Gamma^{(\ell)},s^{(i)}\right> - \left<\Gamma^{(i)},s^{(i)}\right>\right)
\stackrel{\rm (2)}{\leq}  0.
    \end{split}
\end{equation}
Above, the first inequality comes from  \eqref{eqn:proof:of:lemma:3} and the second inequality is based on Definition \ref{ass:single-out}. This completes the proof of Lemma \ref{lem:main:222}. 
\end{proof}
Combining \eqref{eqn:main:1:final} and Lemma \ref{lem:main:222}, we get
\begin{equation*}
    \begin{split}
{} &\P[x_t=x^{(i)}]\\
\leq {} & \exp\left({\eta_t}\|s^{(i)}\|_1\right)\int_{\alpha_t}\II\left[\left\{\argmin\limits_{k\in[K]}\sum_{j=1}^{t-1}f(x^{(k)},y_j)+\left<\Gamma^{(k)},\alpha_t+s^{(i)}\right>\right\} =x^{(i)}\right] p\left(\alpha_t\right)d\alpha_t \\
\leq {} & \exp\left({\eta_t}\|s^{(i)}\|_1\right)\int_{\alpha_t}\II\left[\left\{\argmin\limits_{k\in[K]}\sum_{j=1}^{t}f(x^{(k)},y_j)+\left<\Gamma^{(k)},\alpha_t\right>\right\} =x^{(i)}\right] p\left(\alpha_t\right)d\alpha_t\\
= & {} \exp\left({\eta_t}\|s^{(i)}\|_1\right)\P[x_t'=x^{(i)}]
 \leq \exp(\gamma \eta_t)\P[x_t'=x^{(i)}].
    \end{split}
\end{equation*}

This completes the proof.
\qed

\subsubsection{Proof of Lemma~\ref{lem:main:term1}}
\label{B.2}
We now use Lemma~\ref{lem:main:1111} to prove Lemma~\ref{lem:main:term1}.
Lemma~\ref{lem:main:1111} gives us
\begin{equation}
    \label{eqn:proof:lemma7:first}
 \E[f(x_t,y_t)]\leq \exp\left({\gamma\eta_t}\right)\E[f(x_t',y_t)]
 \leq {} \E[f(x_t',y_t)] + 2{\gamma\eta_t}\E[f(x_t',y_t)],
\end{equation}
Above, the second inequality uses the fact that  ${\gamma\eta_t}\leq 1$ and $\exp(x)\leq 1+2x$ for any $x\in[0,1]$. Next, we focus on bounding the second term in the R.H.S. of \eqref{eqn:proof:lemma7:first}. 
We have
\begin{equation}
    \begin{split}
    \label{eqn:proof:lemma2:temp2}
   {} & {2\gamma}\sum_{t=1}^T \eta_t \E[f(x_t',y_t)] \\
   \stackrel{\rm (1)}{\leq}  {} & {2\gamma}\sum_{t=1}^T \eta_t\E\left[f(x_t',y_t) + \left(\sum_{j=1}^{t-1}f(x_t',y_j)+\left<\Gamma^{(x_t')},\alpha_t\right>\right) -\left(\sum_{j=1}^{t-1}f(x_t,y_j)+\left<\Gamma^{(x_t)},\alpha_t\right>\right)\right] \\
    = {} & 2\gamma\sum_{t=1}^T\eta_t \E\left[ \left(\sum_{j=1}^{t}f(x_t',y_j)+\left<\Gamma^{(x_t')},\alpha_t\right>\right) -\left(\sum_{j=1}^{t-1}f(x_t,y_j)+\left<\Gamma^{(x_t)},\alpha_t\right>\right)\right] \\
    \stackrel{\rm (2)}{\leq} {} &  2\gamma\sum_{t=1}^T\eta_t \E\left[ \left(\sum_{j=1}^{t}f(x_{t+1},y_j)+\left<\Gamma^{(x_{t+1})},\alpha_t\right>\right) -\left(\sum_{j=1}^{t-1}f(x_t,y_j)+\left<\Gamma^{(x_t)},\alpha_{t}\right>\right)\right]\\
     = {} &  2\gamma\sum_{t=1}^T\eta_t \E\left[ \left(\sum_{j=1}^{t}f(x_{t+1},y_j)+\left<\Gamma^{(x_{t+1})},\alpha_{t+1}\right>\right) -\left(\sum_{j=1}^{t-1}f(x_t,y_j)+\left<\Gamma^{(x_t)},\alpha_{t}\right>\right)\right]\\
     {} & +{2\gamma\sum_{t=1}^T \eta_t \left(\frac{1}{\eta_{t}}-\frac{1}{\eta_{t+1}}\right)\E\left[ \Gamma^{(x_{t+1})}\alpha\right]}\\  
    \stackrel{\rm (3)}{\leq}{} &  2\gamma\sum_{t=1}^T\eta_t \E\left[ \left(\sum_{j=1}^{t}f(x_{t+1},y_j)+\left<\Gamma^{(x_{t+1})},\alpha_{t+1}\right>\right) -\left(\sum_{j=1}^{t-1}f(x_t,y_j)+\left<\Gamma^{(x_t)},\alpha_{t}\right>\right)\right]\\
     {} & +{2\gamma\sum_{t=1}^T \eta_t \left(\frac{1}{\eta_{t}}-\frac{1}{\eta_{t+1}}\right)\E\left[\min\limits_{i\in[K]} \Gamma^{(i)}\alpha\right]}\\
     = {} & {2\gamma\eta_T} \cdot\E\left[\sum_{j=1}^T f(x_{T+1},y_j) + \left<\Gamma^{(x_{T+1})},\alpha_{T+1}\right> \right] +{2\gamma\sum_{t=1}^T \eta_t \left(\frac{1}{\eta_{t}}-\frac{1}{\eta_{t+1}}\right)\cdot\E\left[\min\limits_{i\in[K]} \Gamma^{(i)}\alpha\right]}\\
     {} & + 2\gamma\sum_{t=1}^{T-1}(\eta_{t-1}-\eta_t)\cdot\E\left[\sum_{j=1}^{t-1}f(x_t,y_j)+\left<\Gamma^{(x_t)},\alpha_t\right>\right]
    -{2\gamma \eta_1}\cdot\E\left[\min\limits_{i\in[K]} \Gamma^{(i)}\alpha\right]\\
   \stackrel{\rm (4)}{\leq} {} & {2\gamma\eta_T} \cdot\E\left[\sum_{j=1}^T f(x^*,y_j) + \left<\Gamma^{(x^*)},\alpha_{T+1}\right> \right] +{2\gamma\sum_{t=1}^T \eta_t \left(\frac{1}{\eta_{t}}-\frac{1}{\eta_{t+1}}\right)\cdot\E\left[\min\limits_{i\in[K]} \Gamma^{(i)}\alpha\right]}\\
     {} & + 2\gamma\sum_{t=1}^{T-1}(\eta_{t-1}-\eta_t)\cdot\E\left[\sum_{j=1}^{t-1}f(x^*,y_j)+\left<\Gamma^{(x^*)},\alpha_t\right>\right]
    -{2\gamma \eta_1}\cdot\E\left[\min\limits_{i\in[K]} \Gamma^{(i)}\alpha\right]\\
  \stackrel{\rm (5)}{\leq} {} & {2\gamma\eta_TL_T^*} + 2\gamma\sum_{t=1}^{T-1}(\eta_{t-1}-\eta_t)L_{t-1}^*+{2\gamma\sum_{t=1}^T \eta_t \left(\frac{1}{\eta_{t}}-\frac{1}{\eta_{t+1}}\right)\cdot\E\left[\min\limits_{i\in[K]} \Gamma^{(i)}\alpha\right]}\\
 {} & -{2\gamma \eta_1}\cdot\E\left[\min\limits_{i\in[K]} \Gamma^{(i)}\alpha\right].
    \end{split}
\end{equation}
Above, inequality $\rm(1)$ is based on the optimality of $x_t$, inequality $\rm(2)$ is due to the optimality of $x'_t$, inequality $\rm(3)$ is because $\frac{1}{\eta_t}-\frac{1}{\eta_{t+1}}\leq 0$ and $\E[\Gamma^{(x_{t+1})}\alpha]\geq \E[\min_{i\in[K]}\Gamma^{(i)}\alpha]$, inequality $\rm(4)$ is based on the optimality of $x_{t}$, and the final inequality $\rm(5)$ is due to the fact that $x^*$ is independent of $\alpha$ and $\alpha$ is zero-mean. 
Next, we bound each term in the R.H.S. of the above equation respectively. 

For the second term, denote $z_t=\max\{{
\gamma},\frac{1}{c}\sqrt{L_{t-1}^*+1}\}=\frac{1}{\eta_t}.$ WLOG, we assume the upper bound in Lemma~\ref{lem:main:1111} holds for a large enough $\gamma$ such that $\gamma\geq 1$. 
Then, we have
\begin{equation}
    \begin{split}
    \label{eqn:theorem:1:sec:13}
  2\gamma\sum_{t=1}^{T-1}(\eta_{t-1}-\eta_t)L_{t-1}^*\stackrel{\rm (1)}{\leq} {} & 2\gamma\sum_{t=1}^{T-1}\left(\frac{1}{z_{t-1}}-\frac{1}{z_t}\right)z_t^2  \\
  = {} &2\gamma \sum_{t=1}^{T-1} \frac{(z_t-z_{t-1})z_t^2}{z_tz_{t-1}} \\
  = {} & 2\gamma \sum_{t=1}^{T-1} \frac{(z^2_t-z^2_{t-1})z_t}{z_{t-1}(z_t+z_{t-1})} \\
={} & {}  2\gamma \sum_{t=1}^{T-1} \frac{(z^2_t-z^2_{t-1})\left((z_t-z_{t-1})+z_{t-1}\right)}{z_{t-1}(z_t+z_{t-1})} \\
=  {} & 2\gamma \sum_{t=1}^{T-1} \left(\frac{(z^2_t-z^2_{t-1})^2}{z_{t-1}(z_t+z_{t-1})^2} + z_t-z_{t-1}\right) \\
\stackrel{\rm (2)}{\leq}  {} & 2\gamma \sum_{t=1}^{T-1} \left(\frac{(z^2_t-z^2_{t-1})}{z_{t-1}^2} + z_t-z_{t-1}\right) \\
\stackrel{\rm (3)}{\leq} {} &  2\gamma\sum_{t=1}^{T-1} \left(4(\ln(z_t)-\ln(z_{t-1})) + (z_t-z_{t-1})\right) \\
={} & 8\gamma \cdot\ln\left(\frac{z_{T-1}}{z_0}\right)+2\gamma({z_{T-1}}-{z_0})\\
\stackrel{\rm (4)}{\leq} {} &  {8\gamma} \cdot\ln\left(\frac{1}{c}\sqrt{L_T^*+1}+{\gamma}\right)+2\gamma \left(\frac{1}{c}\sqrt{L_T^*+1}+{\gamma}\right).
    \end{split}
\end{equation}
Above, inequality $\rm(2)$ is due to the fact that $z_t+z_{t-1}\geq 1$, $0\leq z^2_t-z^2_{t-1}\leq 1$ and $z_t\geq z_{t-1}$, inequality $\rm(3)$ is based on the identity $x\leq 2\ln(1+x)$ for $x\leq 1$, and the last inequality $\rm(4)$ follows from the definition of $z_t$.

We now control the last two terms of \eqref{eqn:proof:lemma2:temp2}.
Since the distribution of $\alpha$ is symmetric, the distributions of $\alpha$ and $-\alpha$ are the same. Thus, we have 
$$\E\left[\min\limits_{i\in[K]} \Gamma^{(i)}\alpha\right]=\E\left[\min\limits_{i\in[K]} -\Gamma^{(i)}\alpha\right]=-\E\left[\max\limits_{i\in[K]} \Gamma^{(i)}\alpha\right],$$ 
which gives us
\begin{equation}
    \begin{split}
{} & {2\gamma\sum_{t=1}^T \eta_t \left(\frac{1}{\eta_{t}}-\frac{1}{\eta_{t+1}}\right)\cdot\E\left[\min\limits_{i\in[K]} \Gamma^{(i)}\alpha\right]}-{2\gamma \eta_1}\E\left[\min\limits_{i\in[K]} \Gamma^{(i)}\alpha\right] \\
\leq {} & {2\gamma }\max\left\{\E\left[\max\limits_{i\in[K]} \Gamma^{(i)}\alpha\right],0\right\} \left(\frac{\eta_1}{\eta_{T+1}}-1+1\right)\\
\leq  {} & \frac{2}{\eta_{T+1}} \max\left\{\E\left[\max\limits_{i\in[K]} \Gamma^{(i)}\alpha\right],0\right\}.
    \end{split}
\end{equation}
Finally, we leverage the following lemma to complete the proof.
\begin{lemma}
\label{lem:inapp:emax}
We have 
$$ \E\left[\max\limits_{i\in[K]} \Gamma^{(i)}\alpha\right]\leq \sqrt{2}\max\{2\ln K,\sqrt{N\ln K}\}.$$
\end{lemma}
A direct substitution of Lemma~\ref{lem:inapp:emax} obtains the desired control on the last two terms of \eqref{eqn:proof:lemma2:temp2} and completes the proof.
It only remains to prove Lemma~\ref{lem:inapp:emax} which we do below.
\begin{proof}
Let $\beta^{(k)}=\sum_{i=1}^N\Gamma^{(k,i)}\alpha^{(i)}$, and we have for $\lambda <1$, 
\begin{equation}
    \begin{split}
    \label{eqn:theorem:1:exp:max}
      \E\left[\max_{k\in[K]}\Gamma^{(k)}\alpha\right]=   \E\left[\max_{k\in[K]}\beta^{(k)}\right] = {} & \frac{1}{\lambda}\ln\left(\exp\left(\lambda\E\left[\max\limits_{k\in[K]}\beta^{(k)}\right]\right)\right)\\ 
        \stackrel{\rm (1)}{\leq} {} & \frac{1}{\lambda}\ln\left(\E\left[\exp\left(\lambda \max\limits_{k\in[K]}\beta^{(k)}\right)\right]\right)\\
        \stackrel{\rm (2)}{\leq} {} & \frac{1}{\lambda}\ln\left(\sum_{k\in[K]}\E\left[\exp\left(\lambda \beta^{(k)}\right)\right]\right)\\
        = {} & \frac{1}{\lambda}\ln\left(\sum_{k\in[K]}\E\left[\exp\left(\sum_{i=1}^N\lambda \Gamma^{(k,i)}\alpha^{(i)}\right)\right]\right)\\
        \stackrel{\rm (3)}{\leq} {} & \frac{1}{\lambda} 
        \ln\left(K\left(\frac{1}{1-\lambda^2}\right)^{N}\right)\\
        = {} & \frac{\ln K}{\lambda} + \frac{N}{\lambda}\ln\left(\frac{1}{1-\lambda^2}\right)\\
       \stackrel{\rm (4)}{\leq} {} & \min_{\lambda\in\left(0,\frac{\sqrt{2}}{2}\right]}\left[\frac{\ln K}{\lambda} + 2\lambda N\right]\leq \sqrt{2}\max\{2\ln K,\sqrt{N\ln K}\}.
    \end{split}
\end{equation}
Above, inequality $\rm(1)$ is based on Jensen's inequality, inequality $\rm(3)$ follows from the expression of the moment-generating function of a Laplace distribution, and the final inequality $\rm(4)$ follows from the identity $\ln(1/(1-x))\leq 2x$ for $x\in\left(0,\frac{1}{2}\right]$.
This completes the proof.
\end{proof}
\subsubsection{Proof of Lemma \ref{lem:main:3333}}
\label{B.3}
Recall that the infeasible leader is given by
\begin{equation}
    \begin{split}
x_t' = {} & \argmin\limits_{k\in[K]} \sum_{j=1}^t f(x^{(k)},y_j) +\left\langle\Gamma^{(k)},\alpha_t\right\rangle \\
={} &\argmin\limits_{k\in[K]} \sum_{j=1}^t \left(f(x^{(k)},y_j) +\left\langle\Gamma^{(k)},\alpha_j\right\rangle-\left\langle\Gamma^{(k)},\alpha_{j-1}\right\rangle\right).    
    \end{split}
\end{equation}
Recall that we set $\alpha_0=0$. 
Then, we have
\begin{equation}
    \begin{split}
   \sum_{t=1}^Tf(x_t',y_t)+\Gamma^{(x'_t)}\alpha_t-\Gamma^{(x'_t)}\alpha_{t-1}\stackrel{\rm (1)}{\leq} {} & \min\limits_{k\in[K]}\sum_{t=1}^T\left(f(x^{(k)},y_t)+\Gamma^{{(k)}}\alpha_t-\Gamma^{{(k)}}\alpha_{t-1}\right) \\
   = {} & \min\limits_{k\in[K]}\left(\sum_{t=1}^Tf(x^{(k)},y_t)+\Gamma^{{(k)}}\alpha_T\right)\\ 
   \stackrel{\rm (2)}{\leq} {} &\sum_{t=1}^Tf(x^*,y_t)+\Gamma^{{*}}\alpha_T\\ 
   \stackrel{\rm (3)}{\leq} {} & \sum_{t=1}^T f(x^*,y_t) + \max\limits_{k\in[K]}\Gamma^{(k)}\alpha_T,
    \end{split}
\end{equation}
where inequaliy $\rm(1)$ is based on Lemma 3.1 of \citep{bianchi-2006-prediction}.
Because the learning rate sequence $\{\eta_t\}_{t \geq 1}$ is non-increasing, we have $\alpha_{t-1}-\alpha_t\geq 0$.
Thus, we get
\begin{equation}
    \begin{split}
  \sum_{t=1}^T f(x_t',y_t)-f(x^*,y_t)\leq {} & \max\limits_{k\in[K]}\Gamma^{{(k)}}\alpha_T + \sum_{t=1}^T\max\limits_{k\in[K]} \Gamma^{{(k)}}\alpha\cdot \left(\frac{1}{\eta_{t-1}} -\frac{1}{\eta_t}\right).\\ 
    \end{split}
\end{equation}
Taking an expectation on both sides with respect to the randomness in the algorithm yields
\begin{equation}
    \begin{split}
  \E\left[\sum_{t=1}^T f(x_t',y_t)-f(x^*,y_t)\right]\leq {} & \E\left[\max\limits_{k\in[K]}\Gamma^{{(k)}}\alpha_T\right] + \E\left[\sum_{t=1}^T\max\limits_{k\in[K]} \Gamma^{{(k)}}\alpha\cdot\left(\frac{1}{\eta_{t-1}} -\frac{1}{\eta_t}\right)\right]\\
  = {} & \E\left[\max\limits_{k\in[K]}\Gamma^{{(k)}}\alpha_T\right] + \sum_{t=1}^T\left(\frac{1}{\eta_{t}} -\frac{1}{\eta_{t-1}}\right)\cdot\E\left[\max\limits_{k\in[K]} \Gamma^{{(k)}}\alpha\right]\\
  \leq {} & 2\frac{\max\{\E[\max_{i\in[K]}\Gamma^{(i)}\alpha],0\}}{\eta_T},
    \end{split}
\end{equation}
where the first equality follows because the distribution of the Laplace noise is symmetric. The proof is finished by combining the above inequality with Lemma \ref{lem:inapp:emax}. 
\qed
\subsection{Lower Bound for GFTPL}
\label{app:proof:lb}
In this part, we introduce the lower bound for GFTPL. We first prove the following lemma.

\begin{lemma}
\label{lb}
Denote
\begin{equation}
    x_t^*=\argmin\limits_{k\in[K]}\sum_{s=1}^{t}f(x^{(k)},y_s),
\end{equation}
then we have
\begin{equation}
    \Gamma^{(x_{t+1})}\alpha\leq\Gamma^{(x_t^*)}\alpha.
\end{equation}
\end{lemma}
\begin{proof}
Considering the definitions of $x_{t+1}$ and $x_t^*$, we have:
\begin{equation}
    \sum_{s=1}^tf(x_{t+1},y_s)\geq\sum_{s=1}^tf(x_t^*,y_s),
\end{equation}
and
\begin{equation}
    \sum_{s=1}^tf(x_{t}^*,y_s)+\Gamma^{(x_{t}^*)}\frac{\alpha}{\eta_{t+1}}\geq\sum_{s=1}^tf(x_{t+1},y_s)+\Gamma^{(x_{t+1})}\frac{\alpha}{\eta_{t+1}}.
\end{equation}
The required inequality can be shown by adding up the above two inequalities.
\end{proof}
We prove the lower bound result as follows.
\begin{theorem}
\label{thm:main:22}
Assume $\Gamma$ is $\gamma$-approximable, then Algorithm \ref{alg:GFTPL:main} with  $\eta_t=\min\left\{\frac{1}{\gamma},\frac{c}{\sqrt{L^*_{t-1}+1}}\right\}$, where $L_{t-1}^*=\min_{k\in[K]}\sum_{j=1}^{t-1}f(x^{(k)},y_j)$ has the following regret lower bound: 
$$R_T= \E\left[\sum_{t=1}^Tf(x_t,y_t) - \sum_{t=1}^T f(x^*,y_t)\right]\geq -2\sqrt{2}\max\{2\ln K,\sqrt{N\ln K}\}\left({\gamma}+\frac{1}{c}\sqrt{L_T^*+1}\right).
$$
\end{theorem}
\begin{proof}
We first show an intermediate conclusion via induction:
\begin{equation}
\label{lb0}
    \sum_{t=1}^Tf(x_t,y_t)+\Gamma^{(x_t)}\alpha\left(\frac{1}{\eta_{t+1}}-\frac{1}{\eta_t}\right)\geq\sum_{t=1}^Tf(x_{T+1},y_t)+\Gamma^{(x_{T+1})}\frac{\alpha}{\eta_{T+1}}-\Gamma^{(x_{1})}\frac{\alpha}{\eta_{1}},
\end{equation}
which obviously holds for $T=1$. Assume this holds for $T-1$: \begin{equation}
\label{lb1}
     \sum_{t=1}^{T-1}f(x_t,y_t)+\Gamma^{(x_t)}\alpha\left(\frac{1}{\eta_{t+1}}-\frac{1}{\eta_t}\right)\geq\sum_{t=1}^{T-1}f(x_{T},y_t)+\Gamma^{(x_{T})}\frac{\alpha}{\eta_{T}}-\Gamma^{(x_{1})}\frac{\alpha}{\eta_{1}}.
\end{equation}
Noticing
\begin{equation}
    \sum_{t=1}^Tf(x_T,y_t)+\Gamma^{(x_{T})}\frac{\alpha}{\eta_{T+1}}\geq\sum_{t=1}^Tf(x_{T+1},y_t)+\Gamma^{(x_{T+1})}\frac{\alpha}{\eta_{T+1}},
\end{equation}
by rearranging we can  show
\begin{equation}
\label{lb2}
\begin{split}
    f(x_T,y_T)+\Gamma^{(x_{T})}\alpha\left(\frac{1}{\eta_{T+1}}-\frac{1}{\eta_T}\right)\geq&\sum_{t=1}^Tf(x_{T+1},y_t)+\Gamma^{(x_{T+1})}\frac{\alpha}{\eta_{T+1}}\\
    -&\left(\sum_{t=1}^{T-1}f(x_T,y_t)+\Gamma^{(x_{T})}\frac{\alpha}{\eta_{T}}\right)
\end{split}
\end{equation}
Adding up  \eqref{lb1} and \eqref{lb2} we can prove the required conclusion for round $T$.

Combining  \eqref{lb0} and
\begin{equation}
    \sum_{t=1}^Tf(x_{T+1},y_t)\geq\sum_{t=1}^Tf(x_T^*,y_t),
\end{equation}
we have
\begin{equation}
    \begin{split}
        &\sum_{t=1}^T\left(f(x_t,y_t)-f(x^*_T,y_t)\right)\\
        \geq&\sum_{t=1}^T\left(f(x_t,y_t)-f(x_{T+1},y_t)\right)\\
        \geq&\Gamma^{(x_{T+1})}\frac{\alpha}{\eta_{T+1}}-\Gamma^{(x_{1})}\frac{\alpha}{\eta_{1}}-\sum_{t=1}^T\Gamma^{(x_t)}\alpha\left(\frac{1}{\eta_{t+1}}-\frac{1}{\eta_t}\right)\\
        \geq&\Gamma^{(x_{T+1})}\frac{\alpha}{\eta_{T+1}}-\Gamma^{(x_{1})}\frac{\alpha}{\eta_{1}}-\sum_{t=1}^T\Gamma^{(x_{t-1}^*)}\alpha\left(\frac{1}{\eta_{t+1}}-\frac{1}{\eta_t}\right)\\
        \geq&-2\frac{\max\limits_{i\in[K]} \Gamma^{(i)}\alpha}{\eta_{T+1}}-\sum_{t=1}^T\Gamma^{(x_{t-1}^*)}\alpha\left(\frac{1}{\eta_{t+1}}-\frac{1}{\eta_t}\right),
    \end{split}
\end{equation}
where for the third inequality,  Lemma \ref{lb} is adopted
while for the fourth, we use the symmetry of $\alpha$ and the non-increasing property of $\eta_t$. Now we can take the expectation and get
\begin{equation}
\begin{split}
    \E\left[\sum_{t=1}^T\left(f(x_t,y_t)-f(x^*_T,y_t)\right)\right]\geq& -2\E\left[\frac{\max\limits_{i\in[K]} \Gamma^{(i)}\alpha}{\eta_{T+1}}\right]\\
    \geq&-2\sqrt{2}\max\{2\ln K,\sqrt{N\ln K}\}\left({\gamma}+\frac{1}{c}\sqrt{L_T^*+1}\right),
\end{split}
\end{equation}
where we use  Lemma \ref{lem:inapp:emax}, $\E[\alpha]=0$ and  $\eta_t=\min\left\{\frac{1}{\gamma},\frac{c}{\sqrt{L^*_{t-1}+1}}\right\}$.
\end{proof}


\paragraph{Remark} Combining Theorems \ref{thm:main:11} and \ref{thm:main:22} while setting $c=1$, we have

\begin{equation}
    \begin{split}
        -O\left(\frac{\max\left\{\ln K,\sqrt{N\ln K}\right\}}{\sqrt{L_T^*}}\right)\leq\frac{\E\left[\sum_{t=1}^Tf(x_t,y_t)\right]}{L_T^*}-1 
        \leq O\left(\frac{\max\left\{{\gamma},\ln K, \sqrt{N\ln K}\right\}}{\sqrt{L_T^*}}\right).
    \end{split}
\end{equation}
As $L_T^*$ goes to $\infty$, both sides go to $0$, which means our strategy competes the best expert in hindsight.
\subsection{Extension to \texorpdfstring{$\ell_p$}{lp}  Perturbation}
\label{app:proof:ellp}
In this section, we extend our techniques to perturbation distributions that are exponential with respect to an $\ell_p$-norm for any $p \geq 1$ (note that $p = 1$ corresponds to the case of the Laplace distribution).
Specifically, we consider the probability density function
\begin{equation}
    p(\alpha)\propto\exp\left\{-\left(\sum_i |\alpha^{(i)}|^p\right)^{\frac{1}{p}}\right\}.
\end{equation}
Recall that we have the following decomposition of regret:
\begin{equation}
    \begin{split}
R_T= {} &\E\left[\sum_{t=1}^Tf(x_t,y_t)-f(x^*,y_t)\right]\\     
= {} & \underbrace{\E\left[\sum_{t=1}^Tf(x_t,y_t)-\sum_{t=1}^Tf(x'_t,y_t)\right]}_{\textsc{term 1}} + \underbrace{\E\left[\sum_{t=1}^Tf(x'_t,y_t)-\sum_{t=1}^Tf(x^*,y_t)\right]}_{\textsc{term 2}},
    \end{split}
\end{equation}
A critical observation is that the proof of Lemma \ref{lem:main:1111} relies on the triangle inequality
$$\frac{p(\beta)}{p\left(\beta-s^{(i)}\right)}=\exp\left(\eta_t\left(\left\|\beta-s^{(i)}\right\|_1-\|\beta\|_1\right)\right)\leq \exp\left({\eta_t} \|s^{(i)}\|_1\right),$$
which is easily generalized to the $\ell_p$-norm:
$$\frac{p(\beta)}{p\left(\beta-s^{(i)}\right)}=\exp\left(\eta_t\left(\left\|\beta-s^{(i)}\right\|_p-\|\beta\|_p\right)\right)\leq \exp\left({\eta_t} \|s^{(i)}\|_p\right).$$
Following the proof of Lemma ~\ref{lem:main:1111}, we then get
\begin{equation}
 \E[f(x_t,y_t)]\leq {} \exp\left({\gamma_p\eta_t}\right)\E[f(x_t',y_t)],
\end{equation}
where $\gamma_p$ is now an upper bound on $\|s^{(i)}\|_p$.
As before, noting that ${\gamma_p\eta_t}\leq 1$ and $\exp(x)\leq 1+2x$ for any $x\in [0,1]$ gives us
\begin{equation}
   \E[f(x_t,y_t)]\leq {} \exp\left({\gamma_p\eta_t}\right)\E[f(x_t',y_t)]\leq \left(1+{2\gamma_p\eta_t}\right) \E[f(x_t',y_t)].
\end{equation}
It remains to upper bound $\E\left[\max\limits_{i\in[K]} \Gamma^{(i)}\alpha\right]$ under the $\ell_p$ perturbation (as was previously done for the Laplace case).
This is done in the following lemma.
\begin{lemma}
\label{app:lemma:emaxp}
Under the $\ell_p$ perturbation, we have
\begin{equation}
    \E\left[\max\limits_{i\in[K]} \Gamma^{(i)}\alpha\right]\leq 2N^{1-\frac{1}{p}}(\ln K+N\ln 2).
\end{equation}
\end{lemma}
\begin{proof}
Similar to the proof of Lemma ~\ref{lem:inapp:emax}, we have
\begin{equation}
    \E\left[\max\limits_{i\in[K]} \Gamma^{(i)}\alpha\right]\leq\frac{1}{\lambda}\ln \left(K\cdot\E\left[\exp\left(\sum_{i=1}^N\lambda|\alpha^{(i)}|\right)\right]\right),
\end{equation}
where we use the fact that $\Gamma\in [0,1]^{K\times N}$. Now we calculate
\begin{equation}
\begin{split}
    \E\left[\exp\left(\sum_{i=1}^N\lambda|\alpha^{(i)}|\right)\right]=&\frac{\int\exp(\lambda\|\alpha\|_1)\cdot\exp(-\|\alpha\|_p)d\alpha}{\int\exp(-\|\alpha\|_p)}d\alpha\\
    \leq&\frac{\int\exp(-(1-\lambda N^{1-\frac{1  }{p}})\|\alpha\|_p)d\alpha}{\int\exp(-\|\alpha\|_p)d\alpha},
\end{split}
\end{equation}
where the norm inequality $\|\alpha\|_1\leq N^{1-\frac{1}{p}}\|\alpha\|_p$ is used. 
Setting $\lambda = \frac{1}{2N^{1-\frac{1}{p}}}$ gives us
\begin{equation}
    \E\left[\exp\left(\sum_{i=1}^N\lambda|\alpha^{(i)}|\right)\right]\leq 2^N,
\end{equation}
and thus
\begin{equation}
    \E\left[\max\limits_{i\in[K]} \Gamma^{(i)}\alpha\right]\leq 2N^{1-\frac{1}{p}}(\ln K+N\ln 2).
\end{equation}
\end{proof}
We now complete the proof extension.
According to the proof of Lemma 5, for $\textsc{term 1}$ we have
\begin{equation}
\begin{split}
    &\E\left[\sum_{t=1}^Tf(x_t,y_t)-\sum_{t=1}^Tf(x'_t,y_t)\right]\leq {2\gamma_p}\sum_{t=1}^T\eta_t\E[f(x_t',y_t)]\\
    \leq&{2\gamma_p}\left(\eta_TL_T^*+\sum_{t=1}^{T-1}(\eta_{t-1}-\eta_t)L_{t-1}^*+\sum_{t=1}^T\eta_t\left(\frac{1}{\eta_t}-\frac{1}{\eta_{t+1}}\right)\E\left[\min\limits_{i\in[K]}\Gamma^{(i)}\alpha\right] 
     -\eta_1\E\left[\min\limits_{i\in[K]} \Gamma^{(i)}\alpha\right] \right)\\
     \leq &{2\gamma_p}\left(cN^{1-\frac{1}{p}}\sqrt{L_T^*+1}+\left(\frac{1}{cN^{1-\frac{1}{p}}}\sqrt{L_T^*+1}+{\gamma_p}\right)+4\ln\left(\frac{1}{cN^{1-\frac{1}{p}}}\sqrt{L_T^*+1}+{\gamma_p}\right)\right)\\
     +&2\phi\left({\gamma_p}+\frac{\sqrt{L_T^*+1}}{cN^{1-\frac{1}{p}}}\right),
\end{split}
\end{equation}
where $\phi$ denotes an upper bound on $\E\left[\max\limits_{i\in[K]} \Gamma^{(i)}\alpha\right]$ that will be specified shortly.
Above, we plug in $\eta_t=\min\left\{\frac{1}{\gamma_p},\frac{cN^{1-\frac{1}{p}}}{\sqrt{L_{t-1}^*+1}}\right\}$ to get the third inequality.
For $\textsc{term 2}$, a similar argument to the proof of Lemma~\ref{lem:main:3333} gives 
\begin{equation}
    \E\left[\sum_{t=1}^Tf(x'_t,y_t)-\sum_{t=1}^Tf(x^*,y_t)\right]\leq 2\phi\left({\gamma_p}+\frac{\sqrt{L_T^*+1}}{cN^{1-\frac{1}{p}}}\right).
\end{equation}
Thus, the total regret is upper bounded by
\begin{equation}
    \begin{split}
        R_T\leq& {2\gamma_p}\left(cN^{1-\frac{1}{p}}\sqrt{L_T^*+1}+\left(\frac{1}{cN^{1-\frac{1}{p}}}\sqrt{L_T^*+1}+{\gamma_p}\right)+4\ln\left(\frac{1}{cN^{1-\frac{1}{p}}}\sqrt{L_T^*+1}+{\gamma_p}\right)\right)\\
     +&4\phi\left({\gamma_p}+\frac{\sqrt{L_T^*+1}}{cN^{1-\frac{1}{p}}}\right)=O\left({2\gamma_p}\left(cN^{1-\frac{1}{p}}+\frac{1}{cN^{1-\frac{1}{p}}}\right)\sqrt{L_T^*+1}+4\phi\frac{\sqrt{L_T^*+1}}{cN^{1-\frac{1}{p}}}\right).
    \end{split}
\end{equation}

If we use a $\ell_p$ perturbation, by Lemma \ref{app:lemma:emaxp}, we have $\phi=2N^{1-\frac{1}{p}}(\ln K+N\ln 2)$ and
\begin{equation}
\begin{split}
    R_T=&O\left({2\gamma_p}\left(cN^{1-\frac{1}{p}}+\frac{1}{cN^{1-\frac{1}{p}}}\right)\sqrt{L_T^*+1}+8N^{1-\frac{1}{p}}(\ln K+N\ln 2)\frac{\sqrt{L_T^*+1}}{cN^{1-\frac{1}{p}}} \right)\\
    =&O\left(\max\left\{\gamma_pN^{1-\frac{1}{p}},\ln K,N\right\}\sqrt{L_T^*}\right).
\end{split}
\end{equation}
which completes the proof.
\qed

We do a brief comparison between the $\ell_p$-perturbation and Laplace perturbation for the case when $\Gamma\in \{0,1\}^{K\times N}$ is a binary matrix. By Lemma \ref{lem:binarygamma} $\gamma_p=N^{\frac{1}{p}}$ because  $s^
{(i)}\in\{-1,1\}^N$. Then we get that the regret under the $\ell_p$ perturbation is
$$
R_T=O\left(\max\left\{N,\ln K\right\}\sqrt{L_T^*}\right),
$$
while by Theorem \ref{thm:main:11} the regret bound under the Laplace distribution is
$$
R_T=O\left(\max\left\{N,\ln K,\sqrt{N\ln K}\right\}\sqrt{L_T^*}\right).
$$
We can see the regret bounds are the same. Since $\ell_p$ perturbation does not lead to an improvement on the regret bound and the Laplace distribution is easier to sample, we only consider the Laplace distribution in the main paper.

\section{Omitted Proof for Section \ref{imp-real}}
In this section, we provide the omitted proofs for Section \ref{imp-real}.
\subsection{Proof of Lemma \ref{lem:binarygamma}}
\label{app:proof:binaryg}

We begin by proving Lemma~\ref{lem:binarygamma}, which shows that any $\{0,1\}$-valued PTM with distinct rows satisfies $\gamma$-approximability.
We first state the following lemma which introduces a slightly stronger condition for $\gamma$-approximability.

\begin{lemma}
\label{ass:single-out-2}
    Let $\Gamma\in[0,1]^{K\times N}$ be a  matrix, and denote $\Gamma^{(k)}$ as the $k$-th row of $\Gamma$. If $\forall k\in[K]$, $\exists s\in\R^{N}$, $\|s\|_{1}\leq \gamma$, such that $\left<\Gamma^{(k)}, s\right>-\left<\Gamma^{(j)}, s\right>\geq 1$ for all rows $j\not=k$, then $\Gamma$ is $\gamma$-approximable.
\end{lemma}
\begin{proof}
Since $\forall y\in\Y, k,j\in [K]$, $1\geq f(x^{(k)},y)-f(x^{(j)},y)$, it is straightforward to see that the condition in Lemma \ref{ass:single-out-2} is a sufficient condition of Definition \ref{ass:single-out}.
\end{proof}
Next, we construct a $\gamma$-approximable $\Gamma$ based on Lemma  \ref{ass:single-out-2}. Denote $\Gamma^{(k,i)}$ as the $i$-th element of $\Gamma^{(k)}$.
$\forall t>0,k\in[K]$, we set $s^{(k)}=2\Gamma^{(k)}-1$. Since $\forall k\in[N]$, $s^{(k)}\in\{-1,1\}^N$, we have $\|s^{(k)}\|_1\leq N$. On the other hand, $\forall j\not=k$,
\begin{equation}
    \begin{split}
    \label{eqn:proof:for:D1}
\left< \Gamma^{(k)},  s^{(k)}\right>-\left< \Gamma^{(j)},  s^{(k)}\right>=\sum_{i=1}^N (\Gamma^{(k,i)}-\Gamma^{(j,i)})\cdot(2\Gamma^{(k,i)}-1).    
    \end{split}
\end{equation}
For each term $i$ in the R.H.S. of the equality, we have $$(\Gamma^{(k,i)}-\Gamma^{(j,i)})\cdot(2\Gamma^{(k,i)}-1)=
\begin{cases}
0, & \Gamma^{(k,i)}=\Gamma^{(j,i)},\\
1, & \Gamma^{(k,i)}\not=\Gamma^{(j,i)}.
\end{cases}
$$
Note that since every two rows of $\Gamma$ differ by at least one element,  there must exist one $i\in[N]$ such that $(\Gamma^{(k,i)}-\Gamma^{(j,i)})\cdot(2\Gamma^{(k,i)}-1)=1$.
This completes the proof of the lemma.
\qed

Lemma~\ref{lem:binarygamma} is simple but powerful, and can be applied to a broad variety of combinatorial auction problems.
This is detailed next.

\subsection{Auction Problems with a Binary \texorpdfstring{$\Gamma$}{ga}}
\label{app:proof:auction}
Imagine that a seller wants to sell $k$ items (that are either homogeneous or heterogeneous) to $n$ bidders.  
Each bidder has a combinatorial utility function $b^{(i)}:\{0,1\}^k\rightarrow [0,1]$ and we use $b$ to denote the bidding profile vector of all bidders. 
In this work we consider \emph{truthful} auctions, i.e. each bidder is incentivized to report his true valuation $b^{(i)}$ in the unique Bayes-Nash equilibrium of the auction. 
The $i$-th bidder gets an allocation $q^{(i)}(b)\in\{0,1\}^k$ and pays the seller $p^{(i)}(b)$.  Therefore, the utility of the bidder is given by $b^{(i)}(q^{(i)}(b))-p^{(i)}(b)$.
 
An auction $a$ receives the bidding profiles of all bidders and determines how to allocate the items and how much to charge each bidder. We use $r(a,b)\coloneqq\sum_{i=1}^n p^{(i)}(b)$ to denote the revenue yielded by applying auction $a$ to the bidder profile $b$.
We consider a \emph{repeated auction} setting in which the auctioneer faces different bidders on each round.
The bidders may be of very heterogeneous types, so we do not make any assumptions on the bidder profile and assume that it can arbitrarily change from round to round.
More formally: for each round $t=1,\dots, T$, the learner chooses an auction $a_t$ while the adversary chooses a bidder profile $b_t$.
Then, the learner gets to know $b_t$ and receives the revenue $r(a_t,b_t)$. 
The goal of the learner is to compete the revenue earned by the best auction in hindsight. Following \cite{dudik2020oracle}, if the revenue $r(a,b)\in [0,R]$ where $R> 1$, then we can scale all rewards by $\frac{1}{R}$ to ensure all rewards are in $[0,1]$. After applying Algorithm \ref{alg:GFTPL:33}, we  scale the reward back to get the $O(R\sqrt{T-L_T^*})$ regret.
 
 Now we briefly introduce  auction problems that admit a binary-valued TPM $\Gamma$.
 By Lemma~\ref{lem:admis} these are $\gamma$-approximable and by Theorem~\ref{thm:main:11} these admit small-loss bounds.
 
 \paragraph{VCG with bidder-specific reserves} For the standard VCG auction, multiple bidders can be simultaneously served if the allocation $q_*$ maximizes the total social welfare $\sum_{i=1}^nb^{(i)}q^{(i)}_*$. Then the bidder who wins a set of items would pay the externality he imposes on others 
 $$
 p^{(i)}(b)=\max_q\left(\sum_{j\neq i}b^{(j)}q^{(j)}\right)-\sum_{j\neq i}b^{(j)}q^{(j)}_*.
 $$
 The setting we discuss is slightly modified in the sense that we have a  vector $a$ with $i$-th component being the reserve value of the $i$-th bidder. Any bidder whose valuation $b^{(i)}$ is smaller than $a^{(i)}$ will be eliminated. Then, we run the VCG auction for the remaining bidders.
 
 Following \cite{dudik2020oracle}, we discretize reserve prices and use the same $\Gamma$ therein to get the following small-loss bound:
 \begin{theorem}
 \label{thm:vcg}
We consider VCG auction with reserves for the single-item $s$-unit setting, and the set of all feasible auctions is denoted by $
\mathcal{I}$. Denote $R=\max_{a,b} r(a,b)$. Let $\Gamma$ be an $|\mathcal{I}_m|\times n\lceil \log m\rceil$ binary matrix, where $\mathcal{I}_m$ contains auctions in which each reservation price comes from $\left\{\frac{1}{m},\dots,\frac{m}{m}\right\}$, and
 consecutive $\lceil \log m\rceil$ columns correspond to binary encodings of each bidder, then $\Gamma$ is implementable. Running Algorithm \ref{alg:GFTPL:33} with such a $\Gamma$ yields
 \begin{equation}
     \E\left[\max_{a\in\mathcal{I}}\sum_{t=1}^Tr(a,b_t)-\sum_{t=1}^Tr(a_t,b_t) \right]=O\left(nR\sqrt{T-L_T^*}\log (T s)\right).
 \end{equation}
 \end{theorem}
 \begin{proof}
 The implementability of $\Gamma$ follows from Lemma 3.3 of \cite{dudik2020oracle}. Since $\Gamma$ is binary and every two rows are distinct, by Lemma \ref{lem:binarygamma} we know it is $N$-approximable. Using Corollary \ref{cor:imp} we have
 \begin{equation}
 \label{eq:vcg1}
 \begin{split}
     \E\left[\max_{a\in\mathcal{I}_m}\sum_{t=1}^Tr(a,b_t)-\sum_{t=1}^Tr(a_t,b_t) \right]=&O\left(R\cdot\max\{\gamma,\ln K,\sqrt{N\ln K}\}\sqrt{T-L_T^*} \right)\\
     =&O\left(R\cdot\max\{N,\ln K,\sqrt{N\ln K}\}\sqrt{T-L_T^*} \right)\\
     =&O\left(R\cdot N\sqrt{T-L_T^*}\right)\\
     =&O\left(R\cdot n\log m\sqrt{T-L_T^*}\right),
 \end{split}
 \end{equation}
 where we use the facts that $N=n\lceil\log m\rceil$, $K=|\mathcal{I}_m|$, and $N=\Omega(\log K)$ since $\Gamma$ is binary. According to \cite{dudik2020oracle}, the optimal revenue in $\mathcal{I}$ is upper bounded by that of $\mathcal{I}_m$:
\begin{equation}
  \label{eq:vcg2}
 \E\left[\max_{a\in\mathcal{I}}\sum_{t=1}^Tr(a,b_t)-\max_{a\in\mathcal{I}_m}\sum_{t=1}^Tr(a,b_t) \right]\leq \frac{Ts}{m}.
\end{equation}
 Combining  \eqref{eq:vcg1} and \eqref{eq:vcg2} while setting $m=O(T s)$ yield the proposed Theorem.
 \end{proof}
 \paragraph{Envy-free item pricing} Assume there are $k$ different items and we use $a$ to denote the vector of each item's price. Bidders come one by one. The $i$-th bidder greedily chooses a bundle $q^{(i)}\in\{0,1\}^k$ which maximizes his utility $b^{(i)}(q^{(i)})-a\cdot q^{(i)}$ and pays $a\cdot q^{(i)}$. Similar as the VCG with bidder-specific reserves, we also assume each price is discretized in the set $a^{(i)}\in\{\frac{1}{m},\dots,\frac{m}{m}\}$.
 \begin{theorem}
 \label{thm:envy}
We consider envy-free auction for $n$ single-minded bidders and $k$ heterogeneous items with infinite supply. Denote $\mathcal{P}$ to be the set of all possible auctions and $R=\max_{a,b} r(a,b)$. Let $\Gamma$ be an $|\mathcal{P}_m|\times (k\lceil \log m\rceil$ binary matrix, where $\mathcal{P}_m$ contains envy-free item auctions in which all prices come from $\left\{\frac{1}{m},\dots,\frac{m}{m}\right\}$ and
 consecutive $\lceil \log m\rceil$ columns correspond to binary encodings of each item's price. 
 Then, $\Gamma$ is implementable and running Algorithm \ref{alg:GFTPL:33} with this value of $\Gamma$ yields
 \begin{equation}
     \E\left[\max_{a\in\mathcal{P}}\sum_{t=1}^Tr(a,b_t)-\sum_{t=1}^Tr(a_t,b_t) \right]=O\left(kR\sqrt{T-L_T^*}\log (kT)\right).
 \end{equation}
 \end{theorem}
 \begin{proof}
 As noticed in \cite{dudik2020oracle}, we can consider a bidder who has valuation $b^{(i)}$ for the bundle of the $i$-th item and valuations 0 for any other bundles. The revenue of auction $a$ on such a bidder profile is $a^{(i)}\II[b^{(i)}\geq a^{(i)}]$. Similarly, for the VCG auction with reserves $a$, the revenue of a bidder who has a non-zero valuation $b^{(i)}$ would be $a^{(i)}\II[b^{(i)}\geq a^{(i)}]$. Based on the equivalence between  envy-free auction and VCG auction with reserves, we can apply Theorem \ref{thm:vcg} with $n=k$ to get the following bound.
 
 \begin{equation}
 \label{eq:envy1}
     \E\left[\max_{a\in\mathcal{P}_m}\sum_{t=1}^Tr(a,b_t)-\sum_{t=1}^Tr(a_t,b_t) \right] =O\left(R\cdot k\log m\sqrt{T-L_T^*}\right),
 \end{equation}
As also pointed out by \cite{dudik2020oracle}, the optimal revenue of $\mathcal{P}$ would not be much larger than that of $\mathcal{P}_m$ upto a small discretization-related error:
\begin{equation}
  \label{eq:envy2}
 \E\left[\max_{a\in\mathcal{P}}\sum_{t=1}^Tr(a,b_t)-\max_{a\in\mathcal{P}_m}\sum_{t=1}^Tr(a,b_t) \right]\leq \frac{nk^2T}{m}.
\end{equation}
 Combining Equations \ref{eq:envy1} and \ref{eq:envy2} while setting $m=O(k^2T)$ yields the theorem.
 \end{proof}
 \paragraph{Online welfare maximization for multi-unit items} 

In this setting, we wish to allocate $h \gg n$ homogeneous items to $n$ bidders such that $\sum_{i=1}^na^{(i)}=h$. 
Each bidder has a valuation function $b^{(i)}:\N\rightarrow [0,1]$ that maps the number of items he obtains to the utility. We assume $b^{(i)}$ is non-decreasing and $b^{(i)}(0)=0$. The objective is to maximize the total social welfare $\sum_{i=1}^nb^{(i)}(a^{(i)})$.
 
We denote the set of allocations which satisfy $\sum_{i=1}^na^{(i)}=s$ as $\mathcal{X}$. For the offline version of this problem, \cite{dobzinski2010mechanisms} propose a $\frac{1}{2}$-approximation maximal in range (MIR) algorithm, which means maximizing the total social welfare on a set $\mathcal{X}' \subseteq\mathcal{X}$ yields at least $\frac{1}{2}$ of the maximal social welfare on the whole $\mathcal{X}$. 
We now explain the composition of the set of allocations $\mathcal{X}'$. 
We divide $h$ items into $n^2$ bundles of the same size $A=\left\lfloor\frac{h}{n^2}\right\rfloor$ and a possible distinct bundle with size $A'$ which contains all the remaining items. $\mathcal{X}'$ contains all allocations about these $O(n^2)$ bundles in the sense that all items in a bundle can only be simultaneously allocated. Finally the problem is converted to a knapsack problem and there exists an $\frac{1}{2}$-approximation algorithm that runs in $O(\text{poly}(n))$ time.

For the construction of $\Gamma$, we make some modifications to the original construction of \cite{dudik2020oracle} to get a binary-valued PTM. We first define $\mathcal{A}=\{mA+nA':m\in\{0,1,\dots,n^2\},n\in\{0,1\}\}$; note that $|\mathcal{A}| \leq 2n^2 + 2$.
We denote $g_1,\dots,g_{|\mathcal{A}|}$ to be the elements of $\mathcal{A}$ in non-decreasing order. 
Then, we select $\Gamma$ to be a $|\mathcal{X}'|\times n|\mathcal{A}|$ matrix. 
For any allocation $a^{(k)}=[g_{\tau_1},\dots,g_{\tau_n}]$, $k\in [|\mathcal{X}'|]$, $j\in[n]$ and $\ell\in[|\mathcal{A}|]$, we define $\Gamma^{(k,i)}=\II[\tau_j> \ell]$ where $i=(j-1)|\mathcal{A}|+\ell$. 
Note that $\Gamma$ is $1$-implementable because each column corresponds to a valid valuation function. 
In addition, we have the following result that bounds the regret with respect to the $1/2$-approximation of the best revenue in hindsight.
\begin{theorem}
With the aforementioned $\Gamma$ in hand, we can combine Algorithm \ref{alg:GFTPL:33} with the $\frac{1}{2}$-approximate MIR algorithm in \cite{dobzinski2010mechanisms} and get the following regret bound:
\begin{equation}
    \E\left[\frac{1}{2}\left(\max_{a\in\mathcal{X}}\sum_{t=1}^T\sum_{i=1}^nb_t^{(i)}(a^{(i)})\right)-\sum_{t=1}^T\sum_{i=1}^nb_t^{(i)}(a_t^{(i)}) \right]=O(n^3\sqrt{T-L_T^*})
\end{equation}
\end{theorem}
\begin{proof}
We first show that $\Gamma$ is $N$-approximable. Since $\Gamma$ is a binary matrix, by Lemma \ref{lem:binarygamma} it suffices to show that $\Gamma$ does not possess two identical rows.
This can be verified by noticing that $\Gamma^{(k)}$ and $\Gamma^{(k')}$ are binary encodings of $a^{(k)}$ and $a^{(k')}$ by applying indicator functions.

Thus, the PTM $\Gamma$ is indeed $N$-approximable.
By Corollary \ref{cor:imp} we have
\begin{equation}
 \label{eq:mu1}
\begin{split}
    \E\left[\max_{a\in\mathcal{X}'}\left(\sum_{t=1}^T\sum_{i=1}^nb_t^{(i)}(a^{(i)})\right)-\sum_{t=1}^T\sum_{i=1}^nb_t^{(i)}(a_t^{(i)}) \right] &=O\left(\max\{\gamma,\ln K,\sqrt{N\ln K}\}\sqrt{T-L_T^*} \right)\\
     &= O\left(\max\{N,\ln K,\sqrt{N\ln K}\}\sqrt{T-L_T^*} \right)\\
     &= O\left(N\sqrt{T-L_T^*}\right)\\
     &= O\left(n|\mathcal{A}|\sqrt{T-L_T^*}\right)\\
     &= O(n^3\sqrt{T-L_T^*}).
\end{split}
\end{equation}
Above, we use the fact that $\Gamma$ is binary-valued to get $N=\Omega(\log K)$,  $N=n|\mathcal{A|}$ and $|\mathcal{A}|=O(n^2)$.
Combining \eqref{eq:mu1} with the fact that the best allocation in $\mathcal{X}'$ is a $\frac{1}{2}$-approximation to the best allocation in $\mathcal{X}$ yields the stated regret bound.
\end{proof}
\paragraph{Simultaneous second-price auctions} 
We now consider the \emph{utility optimization} problem from the point of view of a \emph{bidder} repeatedly participating in a simultaneous second-price auction (with different bidders each time).
In this problem, $n$ bidders want to bid for $h$ items. Each bidder has a combinatorial valuation function $v$ to describe valuations for different bundles and submits a bid vector $b$ for all $h$ items. If he gets an allocation $q$, his payment profile is given $p$, where $p$ is the vector of the second highest bids.
In particular, his utility is given by $u(b,p)=v(q)-p\cdot q$.
Each round the bidder chooses a bidder vector and the adversary chooses the second largest bidder's vector. The goal is to find bidding vectors which compete with the best bidding vector in hindsight.

Following \cite{dudik2020oracle}, we assume that both bids and the valuation function only take values in the discretized set $\left\{0,\frac{1}{m},\dots,\frac{m}{m}\right\}$. We also make the no-overbiddding assumption that $v(q)\geq p\cdot q$ and denote the set of feasible bidding vectors to be $\mathcal{B}$. Let $\Gamma$ be a $\mathcal{B}\times hm$ matrix. For any $j\in[h]$, $\ell\in[m]$, denote $i=(j-1)m+\ell$. For a bidding vector $b^{(k)}=[b^{(k,1)},\dots,b^{(k,h)}]$ and a vector of the second largest bids $p^{(i)}=\frac{\ell}{m}e_j+\sum_{j'\neq j}e_{j'}$, we set $\Gamma^{(k,i)}=\II\left[b^{(k,j)}\geq\frac{\ell}{m}\right]=\frac{u(b^{(k)},p^{(i)})}{v(e_j)-\frac{\ell}{m}}$.
Note that this directly implies that $\Gamma$ is $1$-implementable.

\begin{theorem}
The aforementioned $\Gamma$ is $N$-approximable. Thus, running Algorithm $\ref{alg:GFTPL:33}$ for the simultaneous second-price auction on the discretized set $\mathcal{B}$ yields
\begin{equation}
    \E\left[\max_{b\in\mathcal{B}}\sum_{t=1}^Tu(b,p_t)-\sum_{t=1}^Tu(b_t,p_t)\right]=O(hm\sqrt{T-L_T^*})
\end{equation}
\end{theorem}
 \begin{proof}
 Notice that $\Gamma$ is binary and the rows of $\Gamma$ come from component-wise threshold functions of the bidding vector. Therefore, for two different bidding vectors the corresponding two rows in $\Gamma$ would also be different. 
 Thus, we can apply Corollary \ref{cor:imp} to get
 \begin{equation}
 \begin{split}
    \E\left[\max_{b\in\mathcal{B}}\sum_{t=1}^Tu(b,p_t)-\sum_{t=1}^Tu(b_t,p_t)\right]=&O\left(\max\{\gamma,\ln K,\sqrt{N\ln K}\}\sqrt{T-L_T^*} \right)\\
     =&O\left(\max\{N,\ln K,\sqrt{N\ln K}\}\sqrt{T-L_T^*} \right)\\
     =&O\left(N\sqrt{T-L_T^*}\right)\\
     =&O(hm\sqrt{T-L_T^*}).
 \end{split}
\end{equation}
This completes the proof.
 \end{proof}
\subsection{Level auction}
\label{app:proof:level}
We consider the online level auction problem with single-item, $n$-bidders, $s$-level and $m$-discretization level. We give only a brief review of this problem setting, and we refer the reader to \citet{dudik2020oracle} for a complete description. In each round $t$ of this problem, firstly an auctioneer picks $s$ non-decreasing thresholds from a discretized set $\{\frac{1}{m},\dots,\frac{m}{m}\}$ for each bidder. Let $a_t=[a_t^{(1,1)},\dots,a_t^{(1,s)},\dots,a_t^{(n,1)},\dots,a_t^{(n,s)}]$ be the collection of the auctioneer's choices, where $a^{(i,j)}$ is $j$-th the threshold for the $i$-th bidder. Let $\A\in\{\frac{1}{m},\dots\frac{m}{m}\}^{ns}$ be all possible auctions. 
We further make the following assumption on $\A$, which corresponds to $\S_{s,m}$ as considered in \cite[Section 3.3]{dudik2020oracle}. 
\begin{ass}
\label{ass:lebel:2}
Assume: (a) $\forall a,{a'}\in\mathcal{A}$, there  exists at least one pair $(i,j)\in [n]\times[s]$, such that $a^{(i,j)}\not={a'}^{(i,j)}$; and (b) $\forall a\in\mathcal{A}$, $\forall i\in[n]$, $a^{(i,1)}<\dots<a^{(i,s)}$.
\end{ass}
After $a_t$ is chosen, the bidders reveal their valuations. We denote the  collection of the bidders' valuations as $b_t=[b_t^{(1)},\dots,b_t^{(n)}]\in\B=[0,1]^n$. As a consequence, the auctioneer obtains a reward $r(a_t,b_t)$, which is calculated based on the following rule. 
\begin{defn}[The rule for level auction]
\label{defn:level}
For each bidder $i\in[n]$, define the level index $\ell_t^{(i)}$ be the maximum $j$ such that $a_t^{(i,j)}\leq b^{(i)}_t$, with $\ell_t^{(i)}=0$ when $a_t^{(i,1)}>b_t^{(i)}$. For each round $t$, all bidders whose level indexes are $0$ would be eliminated.
If no bidder left, $r(a_t,b_t)=0$. Otherwise, the bidder with the highest level index wins the item, and pays the price (i.e., $r(a_t,b_t)$) equal to the
minimum bid that he could have submitted and still won the auction. On the same level, the tie-break rule is in favor of the bidder with the smallest bidding index.
\end{defn}
In this framework the set of ``experts'' is the set $\A$ of threshold configurations over all bidders. The regret of the auctioneer over $T$ iterations is the gap between the revenue generated by the online choice of threshold configurations and the revenue of the best set of thresholds in hindsight, i.e.  
$$R_T := \E\left[\max_{a\in\A}\sum_{t=1}^T r(a,b_t)-\sum_{t=1}^Tr(a_t,b_t)\right].$$
\subsubsection{Algorithm and regret}

We first discuss how to construct an approximable (see Definition \ref{ass:single-out}) and implementable $\Gamma$ with small $\gamma$. Directly constructing $\Gamma$ is rather difficult. Thus, we first consider an \emph{augmented} auction problem with $n+1$ bidders. Let $\mathcal{A'}\in \{\frac{1}{m},\dots,\frac{m}{m}\}^{(n+1)s}$ be the set of possible auctions, and  $\mathcal{B'}\in[0,1]^{n+1}$ be the set of bidder profiles. We construct $\mathcal{A'}$ as follows:
\begin{defn}
\label{ass:lebel}
(a) Distinct auctions (first $n$ bidders): $\forall a,{a'}\in\mathcal{A'}$, there  exists at least one pair $(i,j)\in [n]\times[s]$, such that $a^{(i,j)}\not={a'}^{(i,j)}$; (b) distinct thresholds (first $n$ bidders): $\forall a\in\mathcal{A'}$, $\forall i\in[n]$, $a^{(i,1)}<\dots<a^{(i,s)}$; and (c) fixed thresholds for the $(n+1)$-th bidder: $\forall a\in\mathcal{A'}$, $a^{(n+1,1)}=\frac{1}{m}$, $a^{(n+1,j)}=\frac{j-1}{m}$ for $j\in\{2,\dots,s\}$. 
\end{defn}
Comparing Assumption  \ref{ass:lebel:2} and Definition \ref{ass:lebel}, it can be seen that the elements in $\A'$ and $\A$ have an one-to-one correspondence: $\forall a\in\A$, there only exists one $a'\in\A'$, such that $\forall (i,j)\in[n]\times[s]$, $a^{(i,j)}=a'^{(i,j)}$, and vice versa. 

In the augmented problem, at each round $t$, firstly the auctioneer chooses $a'_t$ from $\A'$. At the same time, we let the bidders reveal $b_t'=[b_t;0]$, where $b_t$ is the bidder vector of the original problem, and $b_t'$ is a $(n+1)$-dimensional vector. Then, the  auctioneer obtain a reward $r'(a'_t,b_t')$, where $r': \A'\times \B'\mapsto [0,1]$ follows the auction rule in Definition \ref{defn:level}. For $a_t'$, denote the related auction in $\A$ as $a_t$, and we have the following lemma.
\begin{lemma}
We have $r(a_t,b_t)=r'(a'_t,b'_t)$, and 
$$\E\left[\max_{a\in\A}\sum_{t=1}^T r(a,b_t)-\sum_{t=1}^Tr(a_t,b_t)\right]=\E\left[\max_{a'\in\A'}\sum_{t=1}^T r'(a',b'_t)-\sum_{t=1}^Tr'(a'_t,b'_t)\right].$$
\end{lemma}
\begin{proof}
Since the last element of $b_t'$ is 0, so the $(n+1)$-th bidder will always be at the 0-th level when computing $r'(a'_t,b'_t)$, it will not change the reward at round $t$. We note that, it does not mean the augmentation is not useful: designing a PTM for $r'(a',b')$ is much easier than for the original problem.
\end{proof}

This lemma reveals a duality between the two problems:
A low-regret and oracle-efficient algorithm for the augmented problem directly induces a low-regret and oracle-efficient  algorithm for the original problem by replacing $a'\in\A'$ with its corresponding auction $a_t$ in $\A$. To help understanding, we illustrate the relationship between the original problem and the augmented problem in Figure \ref{fig:la}.
\begin{figure}
\centering
  \includegraphics[width=0.8\linewidth]{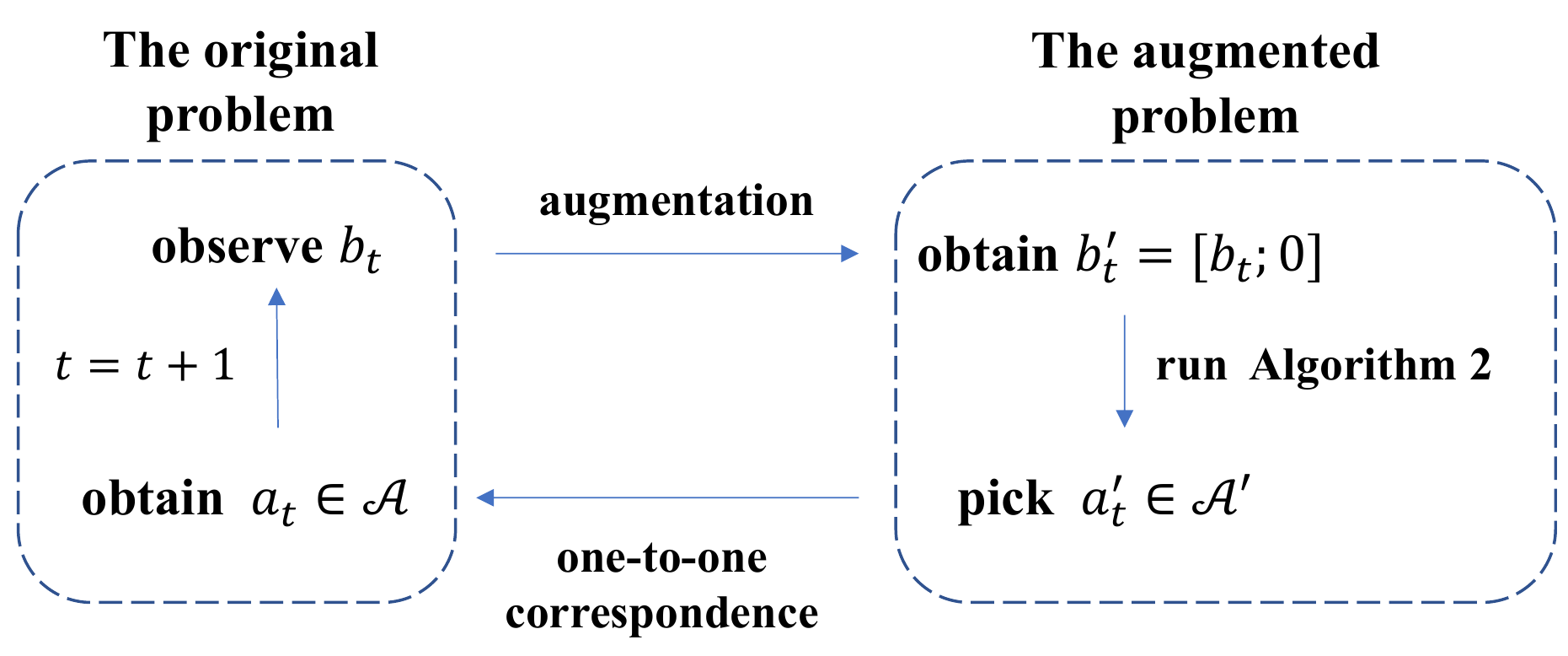}
  \caption{The relationship between the original problem and our proposed augmented problem.}
  \label{fig:la}
\end{figure}

Next, we propose 
to construct an approximable $\Gamma$ matrix for the augmented problem by using  $ns(m-s+1)$ bid profiles in $\B'$. 
For $i\in\{1,\dots,n\}$, $j\in\{1,\dots,s\}$, and  $k\in\{1,\dots,m-s+1\}$, let a bidder vector $b'^{(i,j,k)}\in\B'$ be 
\begin{equation}
\label{eqn:v}
b'^{(i,j,k)}=\frac{k+j-1}{m}e_i+\frac{j-1}{m}e_{n+1},
\end{equation}
where $e_i$ is a $(n+1)$-dimension unit vector whose $i$-th element is 1. Note that the construction here is to ensure that only the $i$-th and the $(n+1)$-th bidders are likely to win, which greatly simplifies our construction of the PTM. 
The following lemma illustrates how to construct the PTM and the corresponding vector $s$ which satisfy the $\gamma$-approximability condition. 
\begin{lemma}
\label{lem:level}
Let $\V=\{b'^{(i,j,k)}\}_{i,j,k}$ be the set containing all $b'$ defined in \eqref{eqn:v}. Let $\Gamma^{\V}\in[0,1]^{K\times ns(m-s+1)}$ be the matrix implemented from $\V$ by assigning each $r'(\cdot,b'^{(i,j,k)})$ to the columns of $\Gamma^{\V}$ one-by-one. Then $\Gamma^{\V}$ is $nsm$-approximable. 
\end{lemma}

\begin{proof}
We first prove $r'(a',b'^{(i,j,k)})=a'^{(i,j)}\II[a'^{(i,j)}\leq \frac{k+j-1}{m}] + \frac{j-1}{m}\II[a'^{(i,j)}> \frac{k+j-1}{m}]$, then argue that it leads to $nsm$-approximability. 


For $j=1$, $b'^{(i,1,k)}=\frac{k}{m}e_i$. According to Definition \ref{defn:level}, when $a'^{(i,1)}>\frac{k}{m}$, all bidders are at the 0-th level, so
no bidder wins the item, and  $r(a',b'^{(i,j,k)})=0$. Otherwise, bidder $i$ wins the item, and pays $a'^{(i,1)}$. Thus, $r'(a',b'^{(i,1,k)})=a'^{(i,1)}\II[a'^{(i,1)}\leq \frac{k}{m}]$. For $j>1$, $b'^{(i,j,k)}=\frac{k+j-1}{m}e_i+\frac{j-1}{m}e_{n+1}$. Based on the third part of Definition \ref{ass:lebel}, bidder $n+1$ is at the $j$-th level. If $a'^{(i,j)}>\frac{k-j+1}{m}$, then bidder $n+1$ wins the item, and pays $r'(a',b'^{(i,j,k)})=\frac{j-1}{m}$. Otherwise, bidder $i$ wins the item, and pays $r'(a',b'^{(i,j,k)})=a'^{(i,j)}$. In sum, we have $r'(a',b'^{(i,j,k)})=a'^{(i,j)}\II[a'^{(i,j)}\leq \frac{k+j-1}{m}] + \frac{j-1}{m}\II[a'^{(i,j)}> \frac{k+j-1}{m}]$ holds for any $j\in [s]$.

Next, we prove the approximability based on Lemma \ref{ass:single-out-2}. WLOG, consider one auction  $a'\in\A'$, and let $\Gamma^{\V,a'}$ be the row related to $a'$, which is a $ns(m-s+1)$-dimensional vector based on Lemma \ref{lem:level}. Our goal is to show that, there exists a vector $s$, such that 
$$ \forall \widehat{a}'\in \A', \widehat{a}'\not=a', \left\langle \Gamma^{\V,a'}-\Gamma^{\V,a},s\right\rangle\geq 1.$$
Denote $s^{(i,j,k)}$ as the  element of $s$ which is related to $b'^{(i,j,k)}$ (see Lemma \ref{lem:level}), and we discuss how to set $s^{(i,j,k)}$ as follows. 

First, based on the second part of Definition \ref{ass:lebel}, we have $\frac{j}{m}\leq a'^{(i,j)}\leq \frac{m-s+j}{m}$. Thus, combining  \eqref{eqn:v}, we know $\forall i\in[n]$,  $j\in[s]$, $\exists k'\in[m-s+1]$, such that $b'^{(i,j,k')}=a'^{(i,j)} e_i+\frac{j-1}{m}e_{n+1}$, which is in turn equivalent to $\frac{k'+j-1}{m}=a'^{(i,j)}$. In the column corresponding to $b'^{(i,j,k')}$, $\forall \widehat{a}'\in\A'$, $\widehat{a}'^{(i,j)}\not=a'^{(i,j)}$, we have $r'({a}',b'^{(i,j,k')})-r'(\widehat{a}',b'^{(i,j,k')})\geq \frac{1}{m}.$ Intuitively, it means that, in this column, only  auctions whose $j$-th threshold for the $i$-th bidder equals to $a'^{(i,j)}$ yield the highest revenue, and these auctions outperform
 other auctions by least $\frac{1}{m}$. Choosing the corresponding  $s^{(i,j,k')}$ as $m$, and setting $s^{(i,j,k)}=0$ for $k\not=k'$, makes $\sum_{k=1}^{m-s+1}r'({a}',b'^{(i,j,k)})s^{(i,j,k)}-r'(\hat{a}',b^{(i,j,k)})s^{(i,j,k)}\geq 1$. Since  we need  to ensure that the loss gap between $a'$ and any other auctions by at least 1, we need to set $s^{(i,j,k')}=m$ for any $i\in[n]$, $j\in [s]$. It is obvious that $\|s\|_1=nsm$ and $\Gamma^{\V}$ is $nsm$-approximable according to Lemma \ref{ass:single-out-2}.


\end{proof}
To illustrate the construction of $\Gamma^{\V}$, we provide an example for the case where $m=5$ and $s=3$ and inspect encodings with respect to a single bidder in Table \ref{tab:exp}. 
\begin{table}[]
\centering
\caption{Illustration of $\Gamma^{\V}$ when $m=5$ and $s=3$. Each $a'^{(i,j)}$ is encoded by $m-s+1=3$ columns. Here $g=\frac{j-1}{m}=\frac{2-1}{5}=\frac{1}{5}$, and $h=\frac{j-1}{m}=\frac{3-1}{5}=\frac{2}{5}$. Consider the auctions whose $a'^{(i,2)}=\frac{3}{5}=\frac{k'+j-1}{m}=\frac{k'+1}{5}$. Then, at the second column  that is related to $a'^{(i,2)}$ (corresponds to $k'=2$), only such auctions can yield the highest revenue $\frac{3}{5}$, and the revenue is at least $\frac{1}{5}$ higher than auctions with $a'^{(i,2)}\neq\frac{3}{5}$.} 
\makeatletter
\label{tab:exp}
\newcommand{\thickhline}{%
    \noalign {\ifnum 0=`}\fi \hrule height 1pt
    \futurelet \reserved@a \@xhline
}

\newcolumntype{"}{@{\hskip\tabcolsep\vrule width 1pt\hskip\tabcolsep}}
\makeatother
\resizebox{\columnwidth}{!}{\begin{tabular}{ccc"cccccccccccc}
\thickhline
  & \textbf{Auction $\mathbf{a'}$} &  &   & & & & \multicolumn{3}{c}{\textbf{Coding}} \\ \cline{5-13} 
\dots  & $(a'^{(i,1)},a'^{(i,2)},a'^{(i,3)})$ & \dots & \multicolumn{1}{l|}{\dots} &  & $a'^{(i,1)}$  &\multicolumn{1}{l|}{}  &  & $a'^{(i,2)}$  &\multicolumn{1}{l|}{} &  & $a'^{(i,3)}$  &\multicolumn{1}{l|}{} & \dots \\ 
\thickhline
\dots  & $(\unitfrac{1}{5},\unitfrac{2}{5},\unitfrac{3}{5})$ & \dots & \multicolumn{1}{l|}{\dots} & \unitfrac{1}{5} & \unitfrac{1}{5} & \multicolumn{1}{l|}{\unitfrac{1}{5}} & \unitfrac{2}{5} & \unitfrac{2}{5} & \multicolumn{1}{l|}{\unitfrac{2}{5}}&  \unitfrac{3}{5}&  \unitfrac{3}{5} & \multicolumn{1}{l|}{ \unitfrac{3}{5}} & \dots\\ \hline 
\dots  & $(\unitfrac{1}{5},\unitfrac{2}{5},\unitfrac{4}{5})$ & \dots & \multicolumn{1}{l|}{\dots} & \unitfrac{1}{5} & \unitfrac{1}{5}& \multicolumn{1}{l|}{\unitfrac{1}{5}} & \unitfrac{2}{5} & \unitfrac{2}{5} & \multicolumn{1}{l|}{\unitfrac{2}{5}}& h&  \unitfrac{4}{5}& \multicolumn{1}{l|}{ \unitfrac{4}{5}}& \dots\\\hline
\dots  &  $(\unitfrac{1}{5},\unitfrac{2}{5},\unitfrac{5}{5})$ & \dots  & \multicolumn{1}{l|}{\dots} & \unitfrac{1}{5} & \unitfrac{1}{5}& \multicolumn{1}{l|}{\unitfrac{1}{5}} & \unitfrac{2}{5} & \unitfrac{2}{5} & \multicolumn{1}{l|}{\unitfrac{2}{5}}& h& h & \multicolumn{1}{l|}{ \unitfrac{5}{5}}& \dots\\\hline
\dots  &  $(\unitfrac{1}{5},\unitfrac{3}{5},\unitfrac{4}{5})$  & \dots  & \multicolumn{1}{l|}{\dots} & \unitfrac{1}{5} &\unitfrac{1}{5}& \multicolumn{1}{l|}{\unitfrac{1}{5}} & g & \unitfrac{3}{5} & \multicolumn{1}{l|}{\unitfrac{3}{5}}& h&  \unitfrac{4}{5}& \multicolumn{1}{l|}{ \unitfrac{4}{5}}& \dots\\\hline
\dots  &  $(\unitfrac{1}{5},\unitfrac{3}{5},\unitfrac{5}{5})$  & \dots  & \multicolumn{1}{l|}{\dots} & \unitfrac{1}{5} &\unitfrac{1}{5}& \multicolumn{1}{l|}{\unitfrac{1}{5}}& g& \unitfrac{3}{5} & \multicolumn{1}{l|}{\unitfrac{3}{5}}& h& h& \multicolumn{1}{l|}{ \unitfrac{5}{5}}& \dots\\\hline
\dots  &  $(\unitfrac{1}{5},\unitfrac{4}{5},\unitfrac{5}{5})$  & \dots  & \multicolumn{1}{l|}{\dots} & \unitfrac{1}{5} &\unitfrac{1}{5}& \multicolumn{1}{l|}{\unitfrac{1}{5}}& g& g& \multicolumn{1}{l|}{\unitfrac{4}{5}}& h& h& \multicolumn{1}{l|}{ \unitfrac{5}{5}}& \dots\\\hline
\dots &  $(\unitfrac{2}{5},\unitfrac{3}{5},\unitfrac{4}{5})$  & \dots  & \multicolumn{1}{l|}{\dots} & 0 &\unitfrac{2}{5}& \multicolumn{1}{l|}{\unitfrac{2}{5}}& g& \unitfrac{3}{5}& \multicolumn{1}{l|}{\unitfrac{3}{5}}& h&  \unitfrac{4}{5}& \multicolumn{1}{l|}{ \unitfrac{4}{5}}& \dots\\\hline
\dots &  $(\unitfrac{2}{5},\unitfrac{3}{5},\unitfrac{5}{5})$  & \dots  & \multicolumn{1}{l|}{\dots} & 0 &\unitfrac{2}{5}& \multicolumn{1}{l|}{\unitfrac{2}{5}}& g& \unitfrac{3}{5}& \multicolumn{1}{l|}{\unitfrac{3}{5}}& h& h& \multicolumn{1}{l|}{ \unitfrac{5}{5}}& \dots\\\hline
\dots &  $(\unitfrac{2}{5},\unitfrac{4}{5},\unitfrac{5}{5})$  & \dots  & \multicolumn{1}{l|}{\dots} & 0 &\unitfrac{2}{5}& \multicolumn{1}{l|}{\unitfrac{2}{5}}& g& g& \multicolumn{1}{l|}{\unitfrac{4}{5}}& h& h& \multicolumn{1}{l|}{ \unitfrac{5}{5}} & \dots\\\hline
\dots &  $(\unitfrac{3}{5},\unitfrac{4}{5},\unitfrac{5}{5})$  & \dots  & \multicolumn{1}{l|}{\dots} & 0 &0& \multicolumn{1}{l|}{\unitfrac{3}{5}}& g & g& \multicolumn{1}{l|}{\unitfrac{4}{5}}& h& h& \multicolumn{1}{l|}{ \unitfrac{5}{5}}& \dots\\\hline
\dots &  \dots  & \dots &\multicolumn{1}{l|}{\dots} & \dots & \dots & \multicolumn{1}{l|}{\dots} & \dots & \dots & \multicolumn{1}{l|}{\dots} & \dots & \dots & \multicolumn{1}{l|}{\dots}& \dots \\
\end{tabular}}
\end{table}

Combining Lemma \ref{lem:level}, Corollary \ref{cor:imp}, we get the following results. 
\begin{cor}
\label{cor:imp:3}
Consider running Algorithm \ref{alg:GFTPL:33} with $\Gamma^{\V}$ on the augmented problem. Let $\{a'_t\}_t\in\A'^T$ be the output of the algorithm, and  $\{a_t\}_t\in\A^T$ be the corresponding auctions in the original problem. Then Algorithm \ref{alg:GFTPL:33} is oracle-efficient, and
\begin{equation}
    \begin{split}
        \E\left[\sum_{t=1}^T r(a^*,b_t)-\sum_{t=1}^Tr(a_t,b_t)\right]= {} &O\left(\max\left\{nsm,ns\ln m, \sqrt{nms\cdot ns\ln m}\right\}\sqrt{T-L_T^*}\right) \\
        = {} & O\left(nsm\sqrt{T-L_T^*}\right).
    \end{split}
\end{equation}
$$$$
\end{cor}

\subsection{Proof of Lemma \ref{lem:small:y}}
\label{app:proof:lem:smally}

In this section, we prove Lemma~\ref{lem:small:y}, which yields oracle-efficient online learning with a very small output space.
Recall that $\mathcal{X}=\{x^{(1)},\dots,x^{(K)}\}$ and we denote the adversary's output space as $\mathcal{Y}=\{y^{(1)},\dots,y^{(d)}\}$. We construct $\Gamma$ as $\forall k\in[K], j\in[d]$, 
$$
\Gamma^{(k,j)}=f(x^{(k)},y^{(j)}).
$$
It is straightforward to see that in this way $\Gamma$ is implementable with complexity 1. 
On the other hand, for each $y_t$, we can find $j_t\in [d]$, such that $y^{(j_t)}=y_t$. Thus, we can meet $1$-approximability by choosing $s=e_{j_t}$ for action $k\in[K]$ in round $t$, where $e_t$ is a unit vector whose $j_t$-th dimension is 1 and all other elements are 0.
This completes the proof.
\qed

\subsection{Proof of Lemma \ref{lem:small}}
\label{app:lem:5}

In this section, we prove Lemma~\ref{lem:small}, which yields oracle-efficient online learning for the transductive online classification problem.
For this problem, we create a PTM $\Gamma$ with $|\W|$ columns, which is configured as  $\Gamma^{(k,j)}=f(x^{(k)},(w^{(j)},1))$, $\forall k\in[K], j\in[|\W|]$.

It is clear that $\Gamma$ is implementable with complexity 1. Next, we prove that $\Gamma$ approximable. Let $w_t$ be the feature vector observed in round $t$. Then there exists $j_t\in[|\W|]$ such that $w^{(j_t)}=w_t$. If $y_t=1$, then the equation of Definition \ref{ass:single-out} holds by setting $s=e_{j_t}$. If $y_t=0$, then the equation can be met by picking $s=-e_{j_t}$. 
This completes the proof of the lemma.
It is worth noting that this choice of $\Gamma$ need not be $\delta$-admissible for any $\delta > 0$.

\subsection{Negative Implementability}
\label{app:negei}

When the oracle only accepts non-negative weights for minimizing the loss (or non-positive weights for maximizing the reward), Algorithms \ref{alg:GFTPL:main} and \ref{alg:GFTPL:33} cannot make use of the oracle directly, since the noise $\alpha$ can be negative for Algorithm \ref{alg:GFTPL:main}, and positive for Algorithm \ref{alg:GFTPL:33}. 
To handle this issue, we consider two solutions: (a) constructing \emph{negative-implementable} PTMs (first defined by~\cite{dudik2020oracle}); (b) replacing the distribution of $\alpha$ from the Laplace distribution with the (negative) exponential distribution. The former solution can be used in  VCG with bidder-specific reserves, envy-free item pricing, problems with small $\Y$ and transductive online classification, while the latter is suitable for the level auction problem and multi-unit online welfare maximization.

\subsubsection{Negative implementable PTM}
To deal with negative weights, \cite{dudik2020oracle} introduce the concept of negative implementability:
\begin{defn}
A matrix $\Gamma$ is negatively implementable with complexity $M$ if for each $j\in[N]$ there exist a (non-negatively) weighted dataset $S_j^{-}$, with $|S_j^{-}|\leq M$, such that $\textstyle \forall i,i'\in[K]$, 
$$-(\Gamma^{(i,j)}-\Gamma^{(i',j)})=\sum_{(w,y) \in \R_{+} \times \Y}w(f(x^{(i)},y)-f(x^{(i')},y)).$$
\end{defn}
Similar to Theorem 5.11 of \cite{dudik2020oracle}, we have the following theorem. 
\begin{theorem}
Suppose the oracle can only accept non-negative weights (for minimizing the loss). If $\Gamma$ is implementable and negative implementable with complexity $M$, then Algorithm \ref{alg:GFTPL:main} can achieve oracle-efficiency with per-round complexity $O(T+NM)$.
\end{theorem}
For VCG with bidder-specific reserves and envy-free item pricing, \citep{dudik2020oracle} show that there exist 
\emph{binary} and admissible PTMs that are implementable and negative implementable. Then, based on Lemma \ref{lem:binarygamma}, this kind of PTMs directly leads to approximable, implementable and negative implementable PTMs, so the oracle-efficiency and the small-loss bound can be achieved for these settings according to the theorem above. 

Moreover, we can also find approximable, implementable and negative-implementable PTMs in the other mentioned applications of a) problems with a small output space $\Y$ and b) transductive online classification. 
For application a) recall that we constructed $\Gamma$ as $\Gamma^{(i,j)}=f(x^{(i)},y^{(j)})$. Then, it is straightforward to verify that this matrix can be negatively implemented by setting  $\Gamma^{(i,j)}=1-f(x^{(i)},y^{(j)})$. 
For application b), recall that we set $\Gamma^{(k,j)}=f(x^{(k)},(w^{(j)},1))$, $\forall k\in[K], j\in[|\W|]$.
This PTM can be negatively implemented by simply setting $\Gamma^{(k,j)}=f(x^{(k)},(w^{(j)},0))$, $\forall k\in[K], j\in[|\W|]$, which flips all elements of the binary matrix (for 0 to 1 and 1 to 0).

\subsubsection{Negative exponential distribution}
\begin{algorithm}[t]
\caption{Oracle-based GFTPL with negative exponential distribution}
\label{alg:GFTPL:44}
\begin{algorithmic}[1]
\STATE \textbf{Input:} Data set $S_j$, $j\in[N]$, that implement a matrix  $\Gamma\in[0,1]^{K\times N}$, $\eta_1=\min\{\frac{1}{\gamma},1\}$.
\STATE Draw IID vector $\widehat{\alpha} \sim \text{Exp}(1)^N$, and let $\alpha=[\alpha^{(1)},\dots,\alpha^{(N)}]=-\widehat{\alpha}$
\FOR{$t=1,\dots,T$}
\STATE Choose 
 $\displaystyle x_t \gets \argmin\limits_{k\in [K]}\sum_{j=1}^{t-1} f(x^{(k)},y_j) + \sum_{i=1}^N\frac{\alpha^{(i)}}{\eta_t}\left[\sum_{(w,y)\in \S_i}w\cdot r(x^{(k)},y)\right]$
\STATE Observe $y_t$
\STATE Compute ${\widehat{L}}_{t}^*=\min\limits_{k\in[K]}\sum_{j=1}^{t}f(x^{(k)},y_j)$ by using the oracle, set $\eta_{t+1} \gets \min\left\{\frac{1}{\gamma},\frac{1}{\sqrt{\widehat{L}_t^*+1}}\right\}$ 
\ENDFOR
\end{algorithmic}
\end{algorithm}

For the other auctions problems that we consider in this paper (i.e. multi-unit mechanisms and level auctions), the PTMs that we constructed are not negative implementable. However, we show that for these cases, Algorithm \ref{alg:GFTPL:33} with a negative exponential distribution is good enough to achieve the small-loss bound.
This adjusted algorithm is summarized in Algorithm \ref{alg:GFTPL:44}. 
Note that Algorithm \ref{alg:GFTPL:44} can directly use the oracle because $\alpha$ is non-positive. 

For Algorithm \ref{alg:GFTPL:44}, we introduce the following theorem. Our key observation is that, in these settings, the approximable vector (i.e., the vector $s$ in Definition \ref{ass:single-out}) is \emph{always element-wise non-negative}, which makes the proof go through when using the negative exponential distribution.

\begin{theorem}
\label{thm:negative}
Let $f(x,y)=1-r(x,y)$. Assume $\Gamma$ is $\gamma$-approximable w.r.t. $f(x,y)$ and implementable with function $r(x,y)$. Moreover, suppose $\forall y\in[\Y]$,  $k\in[K]$, the approximable vector $s$ is element-wise non-negative. Then Algorithm \ref{alg:GFTPL:44}  is oracle-efficient and achieves the following regret bound: 
$$R_T= \E\left[L_T^*-\sum_{t=1}^Tr(x_t,y_t)\right]=O\left(\max\left\{{\gamma},\ln K,N\right\}\sqrt{T-L_T^*}\right).$$
\end{theorem}
\begin{proof}
The proof is similar to that of Theorem \ref{thm:main:11}, and here we only provide the sketch of the proof. For the relation between $\P[x_t=x^{(i)}]$ and $\P[x'_t=x^{(i)}]$, similar to \eqref{eqn:main:1:final}, we still have $\P[x_t=x^{(i)}]\leq \exp(\eta_t \gamma)\P[x'_t=x^{(i)}]$. The difference lies in \eqref{eqn:thm:1:main:distributionmove}:  $\beta$ (corresponds to $\alpha$ in Algorithm \ref{alg:GFTPL:44}) therein has support on the non-positive orthant, since $s$ is non-negative, $\frac{p(\beta)}{p(\beta-s)}$ is always well-defined. Thus, the negative exponential distribution is enough to make the proof go through. Next, for the upper bound of \textsc{term 1}, let $\widehat{\alpha}=-\alpha$ be the exponential distribution.
Similar to the proof of Lemma \ref{lem:main:term1}, we have 
\begin{equation*}
    \begin{split}
        {} & 2\gamma \sum_{t=1}^T \eta_t \E[f(x_t',y_t)]  \\
        \stackrel{\rm (1)}{\leq} {} &  2\gamma\sum_{t=1}^T\eta_t \E\left[ \left(\sum_{j=1}^{t}f(x_{t+1},y_j)+\left<\Gamma^{(x_{t+1})},\alpha_{t+1}\right>\right) -\left(\sum_{j=1}^{t-1}f(x_t,y_j)+\left<\Gamma^{(x_t)},\alpha_{t}\right>\right)\right]\\
     {} & +{2\gamma\sum_{t=1}^T \eta_t \left(\frac{1}{\eta_{t}}-\frac{1}{\eta_{t+1}}\right)\E\left[ \Gamma^{(x_{t+1})}\alpha\right]}\\ 
     \stackrel{\rm (2)}{\leq} {} & 2\gamma\sum_{t=1}^T\eta_t \E\left[ \left(\sum_{j=1}^{t}f(x_{t+1},y_j)+\left<\Gamma^{(x_{t+1})},\alpha_{t+1}\right>\right) -\left(\sum_{j=1}^{t-1}f(x_t,y_j)+\left<\Gamma^{(x_t)},\alpha_{t}\right>\right)\right]\\
     {} & +{2\gamma\sum_{t=1}^T \eta_t \left({\frac{1}{\eta_{t+1}}-\frac{1}{\eta_{t}}}\right)\E\left[\max_{i\in[K]} \Gamma^{i}\widehat{\alpha}\right]}\\ 
     \stackrel{\rm (3)}{\leq} {} & {2\gamma\eta_T} \cdot\E\left[\sum_{j=1}^T f(x^*,y_j) + \left<\Gamma^{(x^*)},\alpha_{T+1}\right> \right] +{2\gamma\sum_{t=1}^T \eta_t \left({\frac{1}{\eta_{t+1}}-\frac{1}{\eta_{t}}}\right)\E\left[\max_{i\in[K]} \Gamma^{i}\widehat{\alpha}\right]}\\
     {} & + 2\gamma\sum_{t=1}^{T-1}(\eta_{t-1}-\eta_t)\cdot\E\left[\sum_{j=1}^{t-1}f(x^*,y_j)+\left<\Gamma^{(x^*)},\alpha_t\right>\right]
    +{2\gamma \eta_1}\cdot\E\left[\max\limits_{i\in[K]} \Gamma^{(i)}\widehat{\alpha}\right]\\
    \stackrel{\rm (4)}{\leq} {} & 2\gamma \eta_T L_T^* + 2\gamma \sum_{t=1}^{T-1} (\eta_{t-1}-\eta_{t})L_{t-1}^{*} + 2\gamma \sum_{t=1}^{T-1}(\eta_{t-1}-\eta_{t})\E\left[\max_{i\in[K]} \Gamma^{i}\widehat{\alpha}\right]\\
    {} & +{2\gamma\sum_{t=1}^T \eta_t \left({\frac{1}{\eta_{t+1}}-\frac{1}{\eta_{t}}}\right)\E\left[\max_{i\in[K]} \Gamma^{i}\widehat{\alpha}\right]}+4\gamma \eta_1 \E\left[\max_{i\in[K]} \Gamma^{i}\widehat{\alpha}\right],
    \end{split}
\end{equation*}
where inequality $\rm(1)$ follows from \eqref{eqn:proof:lemma2:temp2}, inequality $\rm(2)$ is because $\widehat{\alpha}=-\alpha$ and $\widehat{\alpha}$ is non-negative, inequaliry $\rm(3)$ is due the the optimality of $x_t$, and the final inequality $\rm(4)$ is based on the non-negativity of $\widehat{\alpha}$. Note that there are some extra terms since the distribution is no longer zero-mean. To proceed, the second term can be upper bounded by \eqref{eqn:theorem:1:sec:13}. For the third term, similar to \eqref{eqn:theorem:1:exp:max}, we have $\forall \lambda\leq 1/2$, 
\begin{equation}
    \begin{split}
    \label{eqn:negative:imp:expt}
   \E\left[\max_{i\in[K]} \Gamma^{i}\widehat{\alpha}\right]\leq  \frac{1}{\lambda} 
        \ln\left(K\left(\frac{1}{1-\lambda}\right)^{N}\right)
        = \frac{\ln K}{\lambda} + \frac{N}{\lambda}\ln\left(\frac{1}{1-\lambda}\right)        
       \leq 4\max\{\ln K,N\},
    \end{split}
\end{equation}
where in the first inequality we make use of the moment generating function of the exponential distribution. Thus, we have 
$$2\gamma \sum_{t=1}^{T-1}(\eta_{t-1}-\eta_t)\E\left[\max_{i\in[K]} \Gamma^{i}\widehat{\alpha}\right]\leq 8\max\{\ln K,N\},$$
and 
$${2\gamma\sum_{t=1}^T \eta_t \left({\frac{1}{\eta_{t+1}}-\frac{1}{\eta_{t}}}\right)\E\left[\max_{i\in[K]} \Gamma^{i}\widehat{\alpha}\right]}\leq \frac{8\max\{\ln K,N\}}{\eta_{T+1}}.$$
The proof can be finished by applying \eqref{eqn:negative:imp:expt} and similar techniques as in the proof of Lemma \ref{lem:inapp:emax} to bound \textsc{term 2}. 
\end{proof}
Finally, we note that, as shown in the proof of  Lemma \ref{lem:level}, the approximable vector for the level auction problem is element-wise non-negative (0-1 vector), so Theorem \ref{thm:negative} can be directly  applied. In the following, we show that 
this conclusion can also be applied to the online welfare maximization for multi-unit items. 
\begin{lemma}
For multi-unit online welfare maximization, there exists an approximable vector $s$ with non-negative entries $\forall a^{(k)}\in\mathcal{X}', k\in[K]$.
\end{lemma}
\begin{proof}
By Lemma \ref{ass:single-out-2}, it suffices to prove that for any $k\in[K]$ there exists non-negative $s$ such that
$
 \left<\Gamma^{(k)} - \Gamma^{(j)}, s \right> \geq 1.
$
For multi-unit online welfare maximization (as illustrated in Appendix \ref{app:proof:auction}), all $h$ items need to be allocated. There do not exist two rows $\Gamma^{(k)}$ and $\Gamma^{(k')}$ such that $\Gamma^{(k)}\preceq\Gamma^{(k')}$ because the corresponding allocations $a^{(k)}$ and $a^{(k')}$ also preserve this partial order relation, which means for allocation $a^{(k)}$ there are unassigned items. We can simply take $s=\Gamma^{(k)}$. Based on the aforementioned observation, there exists at least one index $\ell\in [N]$ such that $\Gamma^{(k,\ell)}=1$ and $\Gamma^{(j,\ell)}=0$, and thus
$
\textstyle \left<\Gamma^{(k)} - \Gamma^{(j)}, s \right> \geq 1.
$

\end{proof}

\section{Proof of Theorem \ref{thm:off}}
\label{E}

In this section we prove Theorem~\ref{thm:off}, which is our oracle-efficient ``best-of-both-worlds" bound, assuming that the adversary is oblivious. 
Then, based on the definitions, we know  $\widehat{U}_T^{\FTL}$ and  $\widehat{U}_T^{\FPL}$
only depend on the adversary (i.e., the past losses), and is independent of the randomness of the algorithm. 
This also applies to  $\I_{T}^{\FTL}$ and $\I_{T}^{\FPL}$.

The regret can be decomposed into two parts:
\begin{equation}
\begin{split}
    R_T^{\OFF} = {} & \E\left[\sum_{t=1}^T f(x_t,y_t) - \sum_{t=1}^T f(x^*,y_t)\right] \\
    \leq  {} & \E\left[\sum_{t\in\I_T^{\FTL}}f(x_t,y_t) - f(x_T^{\FTL,*})\right]+ \E\left[\sum_{t\in\I_T^{\FPL}}f(x_t,y_t) - f(x_T^{\FPL,*})\right]\\
    \leq {} & \widehat{U}_T^{\FTL} +  \widehat{U}_T^{\FPL},
\end{split}
\end{equation}
where $$x_T^{\FTL,*}=\argmin\limits_{i\in[K]} \sum_{t\in\I_T^{\FTL}}f(x^{(i)},y_t),$$
and 
 $$x_T^{\FPL,*}=\argmin\limits_{i\in[K]} \sum_{t\in\I_T^{\FPL}}f(x^{(i)},y_t).$$
  
In round $T$, there are four possible cases:
\begin{itemize}
    \item Case 1: $\Alg_T=\FTL$, and $\Alg_{T+1}=\FTL$.\\
Since after round $T$, the algorithm does not switch, we have
$\widehat{U}_T^{\FTL}\leq \alpha \widehat{U}_T^{\FPL}$ based on lines 4-8. On the other hand, let $t'$ be the last round where the algorithm performs $\FPL$, that is, $\Alg_{t'+1}=\FTL$. Then, in round $t'-1$, if we do the switch $\FTL\rightarrow\FPL$,
then  
$$\alpha\widehat{U}_{t'-1}^{\FPL}\leq \widehat{U}_{t'-1}^{\FTL}.$$
Moreover, note that $\widehat{U}_{t}^{\FTL}$ and $\widehat{U}_{t}^{\FPL}$ are non-decreasing, and also $\widehat{U}_{t'-1}^{\FPL}\geq \widehat{U}_{t'}^{\FPL}-\tau$. Combining with the fact that $\widehat{U}_{t'}^{\FPL}=\widehat{U}_{T}^{\FPL}$ (since we do not feed losses to \FPL\ from round $t'$ to $T$), we have 
$$\alpha(\widehat{U}_T^{\FPL}-\tau)= \alpha(\widehat{U}_{t'}^{\FPL}-\tau)\leq \alpha \widehat{U}_{t'-1}^{\FPL}\leq \widehat{U}_{t'-1}^{\FTL}\leq \widehat{U}_T^{\FTL},$$
so 
$$\widehat{U}_T^{\FPL}\leq \frac{1}{\alpha}\widehat{U}_T^{\FTL}+\tau. $$
If in round $t'-1$ we use \FPL\ and do not switch, then we have 
$$\frac{1}{\beta}\widehat{U}_{t'-1}^{\FPL}\leq \widehat{U}_{t'-1}^{\FTL},$$
thus
$$\frac{1}{\beta}(\widehat{U}_T^{\FPL}-\tau)= \frac{1}{\beta}(\widehat{U}_{t'}^{\FPL}-\tau)\leq \frac{1}{\beta} \widehat{U}_{t'-1}^{\FPL}\leq \widehat{U}_{t'-1}^{\FTL}\leq \widehat{U}_T^{\FTL},$$
which implies that 
$$ \widehat{U}_T^{\FPL}\leq \beta\widehat{U}_T^{\FTL}+\tau. $$
\\
\item Case 2: $\Alg_T=\FTL$, and $\Alg_{T+1}=\FPL$. \\
Since after round $T$, we have $\FTL\rightarrow\FPL$,
we get $\widehat{U}_T^{\FTL}> \alpha \widehat{U}_T^{\FPL}$ based on lines 4-8. 
On the other hand, We know that $\Alg_{T}=\FTL$, so 
after round $T-1$, there are 2 possibilities: 1) the algorithm remains to be $\FTL$. For this case, we have 
$$ \widehat{U}^{\FTL}_{T}-1\leq \widehat{U}^{\FTL}_{T-1}\leq \alpha \widehat{U}_{T-1}^{\FPL}\leq \alpha \widehat{U}_{T}^{\FPL},$$
where we use the fact that the mixability gap $\delta_t\leq 1$.
It yields
$$\widehat{U}_{T}^{\FTL}\leq \alpha \widehat{U}_T^{\FPL}+1. $$
2) The algorithm switches from $\FPL\rightarrow\FTL$. For this case, we have 
$$ \widehat{U}^{\FTL}_{T}-1\leq \widehat{U}^{\FTL}_{T-1}\leq \frac{1}{\beta} \widehat{U}_{T-1}^{\FPL}\leq \frac{1}{\beta} \widehat{U}_{T}^{\FPL},$$
so
$$\widehat{U}_{T}^{\FTL}\leq \frac{1}{\beta} \widehat{U}_T^{\FPL}+1. $$
$$ $$
\item Case 3: $\Alg_T=\FPL$, and $\Alg_{T+1}=\FTL$.\\
Since after round $T$, we switch from $\FPL\rightarrow\FTL$, we have $\widehat{U}_T^{\FTL}\leq \frac{1}{\beta}\widehat{U}_T^{\FPL}$. On the other hand, in round $T-1$, there are 2 cases: 1) After round $T-1$, we switch the algorithm: $\FTL\rightarrow\FPL$. Thus, 
$$\widehat{U}_{T}^{\FTL}\geq\widehat{U}_{T-1}^{\FTL}\geq \alpha \widehat{U}_{T-1}^{\FPL}\geq \alpha( \widehat{U}^{\FPL}_{T}-\tau),$$ 
implying that 
$$ \widehat{U}_T^{\FPL}\leq \frac{1}{\alpha}\widehat{U}_T^{\FTL}+\tau.$$
2) After round $T-1$, the algorithm does not switch: $\FPL\rightarrow\FPL$. Thus, 
$$\frac{1}{\beta}(\widehat{U}_{T}^{\FPL}-\tau)\leq \frac{1}{\beta}\widehat{U}_{T-1}^{\FPL}\leq \widehat{U}_{T-1}^{\FTL}\leq \widehat{U}_{T}^{\FTL},$$
so $\widehat{U}_T^{\FPL}\leq \beta \widehat{U}_T^{\FTL}+\tau$.
\item Case 4: $\Alg_T=\FPL$, $\Alg_{T+1}=\FPL$.\\
For this case, after round $T$, we have 
$$\widehat{U}_T^{\FPL}\leq \beta \widehat{U}_T^{\FTL}.$$
On the other hand, let $t'$ be the last round of the algorithm that plays $\FTL$. So in round $t'-1$, there are 2 possible cases: 1) After $t'-1$, we switch from $\FPL\rightarrow\FTL$. In this case, we have:
$$ \beta(\widehat{U}^{\FTL}_{T}-1)=\beta(\widehat{U}^{\FTL}_{t'}-1)\leq\beta\widehat{U}_{t'-1}^{\FTL}\leq \widehat{U}_{t'-1}^{\FPL}\leq\widehat{U}_{T}^{\FPL},$$
so 
$$\widehat{U}_T^{\FTL}\leq \frac{1}{\beta}\widehat{U}_T^{\FPL}+1.$$
2) After round $t'-1$, we still play $\FTL$. Then, we have 
$$ \widehat{U}_T^{\FTL}-1 =\widehat{U}_{t'}^{\FTL}-1\leq\widehat{U}_{t'-1}^{\FTL}\leq \alpha \widehat{U}_{t'-1}^{\FPL}\leq \alpha\widehat{U}_{T}^{\FPL},$$
so 
$$\widehat{U}_T^{\FTL}\leq \alpha\widehat{U}_{T}^{\FPL}+1.$$
\end{itemize}
Combining all of the pieces, we always have 
$$ \widehat{U}_T^{\FTL}\leq (\alpha+\frac{1}{\beta})\widehat{U}_T^{\FPL}+1,$$
and 
$$ \widehat{U}_T^{\FPL}\leq (\frac{1}{\alpha}+{\beta})\widehat{U}_T^{\FTL}+\tau.$$
Finally, note that based on the definition, it is straightforward to get $\widehat{U}_T^{\FPL}\leq{U}_T^{\FPL}$ and $\widehat{U}_T^{\FTL}\leq{U}_T^{\FTL}$.
Therefore, 
we have 
$$ R_T^{\OFF}\leq \min\left\{\left(1+\alpha+\frac{1}{\beta}\right)U_T^{\FPL}+1, \left(1+\frac{1}{\alpha}+\beta\right){U}_T^{\FTL}+\tau\right\}.$$
Setting $\alpha=\beta=1$ yields the required theorem.
\qed

\end{appendix}

\end{document}